%% file: paper.tex
\documentclass[a4paper,12pt]{article}

\usepackage[unicode=true,hypertexnames=false]{hyperref}
\usepackage{amsmath,amssymb,amsthm,amsfonts}
\usepackage{xspace}
\usepackage{xcolor}
\usepackage{cite}
\usepackage{algorithm,algpseudocode}
\usepackage{marginnote}

\usepackage{tikz}
\usetikzlibrary{external}
\usepackage{pgfplots}
\usepgfplotslibrary{fillbetween}
\pgfplotsset{compat=1.14}

\newtheorem{theorem}{Theorem}
\newtheorem{lemma}{Lemma}
\newtheorem{corollary}{Corollary}
\newtheorem{definition}{Definition}

\newcommand{\ifcompileescapeplots}[1]{#1}

\newcommand{\merk}[1]{}

\newcommand{\revisepar}[2]{#1}
\newcommand{\startnewtext}{}
\newcommand{\finishnewtext}{}

\newcommand{\CC}{{\mathcal C}}
\newcommand{\Z}{{\mathbb Z}}
\newcommand{\oea}{$(1+1)$~EA\xspace}
\newcommand{\lea}{$(1+\lambda)$~EA\xspace}
\newcommand{\eps}{\varepsilon}
\newcommand{\onemax}{\textsc{OneMax}\xspace}
\newcommand{\onell}{(1+(\lambda,\lambda))}
\newcommand{\ga}{$\onell$~GA\xspace}
\newcommand{\lambdabound}{\overline{\lambda}}
\newcommand{\lambdaboundsq}{\overline{\lambda^2}}
\DeclareMathOperator{\avg}{avg}
\DeclareMathOperator{\Var}{Var}

\input{parts/pgfplots-styles.tex}
\input{pic/tikz-plots.tex}
\input{pic/escape.tex}

\begin{document}

\title{Runtime Analysis\\of the $\onell$ Genetic Algorithm\\on Random Satisfiable 3-CNF Formulas\footnote{An extended abstract of this report will appear in the proceedings of the 2017 Genetic and 
Evolutionary Computation Conference (GECCO 2017).}}

\author{Maxim Buzdalov \and Benjamin Doerr}

\maketitle

\begin{abstract}
\input{parts/0-abstract.tex}
\end{abstract}

\input{parts/1-introduction.tex}
\input{parts/2-preliminaries.tex}
\input{parts/3-onell-on-sat.tex}
\input{parts/4-bounded-onell.tex}
\input{parts/5-experiments.tex}
\input{parts/6-conclusion.tex}
\input{parts/7-acknowledgments.tex}

\bibliographystyle{abbrv}
\bibliography{../../../../bibliography}

\end{document}

%% file: parts/pgfplots-styles.tex
\pgfplotscreateplotcyclelist{my custom}{%
blue!60!black,every mark/.append style={fill=blue!60!black},mark=*\\%1
red!70!black,every mark/.append style={fill=red!70!black},mark=*\\%2
green!70!black,every mark/.append style={fill=green!70!black},mark=*\\%3
gray,every mark/.append style={fill=gray},mark=*\\%4
orange,every mark/.append style={fill=orange},mark=*\\%5
brown!80!black,every mark/.append style={fill=brown!80!black},mark=*\\%6
green,every mark/.append style={fill=green},mark=*\\%7
violet!80!black,every mark/.append style={fill=violet!80!black},mark=*\\%8
black,every mark/.append style={fill=black},mark=*\\%9
teal,every mark/.append style={fill=teal},mark=*\\%10
magenta!70!black,every mark/.append style={fill=magenta!70!black},mark=*\\%11
yellow!90!black,every mark/.append style={fill=yellow!90!black},mark=*\\%12
blue!60!black,every mark/.append style={fill=blue!60!black},mark=star\\%13
red!70!black,every mark/.append style={fill=red!70!black},mark=star\\%14
green!70!black,every mark/.append style={fill=green!70!black},mark=star\\%15
gray,every mark/.append style={fill=gray},mark=star\\%16
orange,every mark/.append style={fill=orange},mark=star\\%17
brown!80!black,every mark/.append style={fill=brown!80!black},mark=star\\%18
green,every mark/.append style={fill=green},mark=star\\%19
violet!80!black,every mark/.append style={fill=violet!80!black},mark=star\\%20
black,every mark/.append style={fill=black},mark=star\\%21
teal,every mark/.append style={fill=teal},mark=star\\%22
magenta!70!black,every mark/.append style={fill=magenta!70!black},mark=star\\%23
yellow!90!black,every mark/.append style={fill=yellow!90!black},mark=star\\%24
}

\pgfplotscreateplotcyclelist{escape-style}{blue!60!black,red!70!black,green!70!black,gray,orange,black}

%% file: pic/tikz-plots.tex
\newcommand{\iqrPlotRandomCNF}[3]{
  \begin{tikzpicture}[#3]
    \begin{axis}[enlargelimits=false, xmode=log, log basis x = 2, cycle list name = my custom, xlabel=Problem size, ylabel=Evaluations / problem size, width=#1, height=#2, legend pos=outer north east]
      \addplot+[mark size=1.2] coordinates {(128.0, 15.2734375)(256.0, 17.48828125)(512.0, 18.685546875)(1024.0, 21.2685546875)(2048.0, 23.02490234375)(4096.0, 24.973388671875)(8192.0, 26.8375244140625)(16384.0, 28.60919189453125)(32768.0, 30.261505126953125)(65536.0, 32.46919250488281)(131072.0, 34.522117614746094)(262144.0, 36.58400344848633)(524288.0, 38.41142463684082)(1048576.0, 40.808491706848145)};
      \addlegendentry{$\lambda = 2$};
      \addplot+[mark size=1.2] coordinates {(128.0, 13.2265625)(256.0, 13.87890625)(512.0, 15.3828125)(1024.0, 16.123046875)(2048.0, 17.08056640625)(4096.0, 19.6357421875)(8192.0, 20.4588623046875)(16384.0, 21.82696533203125)(32768.0, 23.29736328125)(65536.0, 24.456588745117188)(131072.0, 25.772239685058594)(262144.0, 27.28072738647461)(524288.0, 27.757349014282227)(1048576.0, 29.657891273498535)};
      \addlegendentry{$\lambda = 3$};
      \addplot+[mark size=1.2] coordinates {(128.0, 12.8828125)(256.0, 13.51953125)(512.0, 13.791015625)(1024.0, 15.0634765625)(2048.0, 15.48876953125)(4096.0, 16.636962890625)(8192.0, 17.4014892578125)(16384.0, 18.15777587890625)(32768.0, 19.543365478515625)(65536.0, 20.250076293945312)(131072.0, 21.05902862548828)(262144.0, 22.220783233642578)(524288.0, 23.228036880493164)(1048576.0, 24.806714057922363)};
      \addlegendentry{$\lambda = 4$};
      \addplot+[mark size=1.2] coordinates {(128.0, 12.8203125)(256.0, 12.875)(512.0, 13.56640625)(1024.0, 13.90234375)(2048.0, 14.26318359375)(4096.0, 14.92822265625)(8192.0, 15.7020263671875)(16384.0, 16.74383544921875)(32768.0, 17.60760498046875)(65536.0, 18.342987060546875)(131072.0, 19.197776794433594)(262144.0, 19.532474517822266)(524288.0, 20.681325912475586)(1048576.0, 21.706171989440918)};
      \addlegendentry{$\lambda = 5$};
      \addplot+[mark size=1.2] coordinates {(128.0, 12.6171875)(256.0, 13.01171875)(512.0, 13.220703125)(1024.0, 14.1396484375)(2048.0, 14.10693359375)(4096.0, 14.490478515625)(8192.0, 15.2103271484375)(16384.0, 16.07489013671875)(32768.0, 16.889129638671875)(65536.0, 17.449234008789062)(131072.0, 17.91065216064453)(262144.0, 18.774242401123047)(524288.0, 19.327791213989258)(1048576.0, 19.89974308013916)};
      \addlegendentry{$\lambda = 6$};
      \addplot+[mark size=1.2] coordinates {(128.0, 13.1328125)(256.0, 13.51171875)(512.0, 13.6328125)(1024.0, 13.802734375)(2048.0, 14.45166015625)(4096.0, 14.58642578125)(8192.0, 15.5211181640625)(16384.0, 15.91070556640625)(32768.0, 16.35009765625)(65536.0, 16.6798095703125)(131072.0, 17.267372131347656)(262144.0, 17.83985137939453)(524288.0, 18.414669036865234)(1048576.0, 18.988052368164062)};
      \addlegendentry{$\lambda = 7$};
      \addplot+[mark size=1.2] coordinates {(128.0, 13.1328125)(256.0, 14.22265625)(512.0, 13.970703125)(1024.0, 14.1337890625)(2048.0, 14.68017578125)(4096.0, 15.017822265625)(8192.0, 15.5167236328125)(16384.0, 16.03961181640625)(32768.0, 16.325714111328125)(65536.0, 16.434097290039062)(131072.0, 17.25147247314453)(262144.0, 17.69626235961914)(524288.0, 17.869600296020508)(1048576.0, 18.64584445953369)};
      \addlegendentry{$\lambda = 8$};
      \addplot+[mark size=1.2] coordinates {(128.0, 14.7734375)(256.0, 14.2421875)(512.0, 14.36328125)(1024.0, 14.458984375)(2048.0, 14.78369140625)(4096.0, 15.003173828125)(8192.0, 15.6732177734375)(16384.0, 15.7989501953125)(32768.0, 16.258697509765625)(65536.0, 16.725189208984375)(131072.0, 17.11279296875)(262144.0, 17.22182846069336)(524288.0, 17.948469161987305)(1048576.0, 18.500608444213867)};
      \addlegendentry{$\lambda = 9$};
      \addplot+[mark size=1.2] coordinates {(128.0, 15.0078125)(256.0, 15.27734375)(512.0, 14.962890625)(1024.0, 15.4306640625)(2048.0, 15.35205078125)(4096.0, 15.815673828125)(8192.0, 15.9998779296875)(16384.0, 16.37945556640625)(32768.0, 16.557037353515625)(65536.0, 16.795822143554688)(131072.0, 17.18262481689453)(262144.0, 17.640575408935547)(524288.0, 18.015310287475586)(1048576.0, 18.210774421691895)};
      \addlegendentry{$\lambda = 10$};
      \addplot+[mark size=1.2] coordinates {(128.0, 16.6796875)(256.0, 15.81640625)(512.0, 15.728515625)(1024.0, 15.931640625)(2048.0, 15.77001953125)(4096.0, 16.1484375)(8192.0, 16.29736328125)(16384.0, 16.74981689453125)(32768.0, 16.942474365234375)(65536.0, 17.199935913085938)(131072.0, 17.57305145263672)(262144.0, 17.661293029785156)(524288.0, 18.310752868652344)(1048576.0, 18.53842544555664)};
      \addlegendentry{$\lambda = 11$};
      \addplot+[mark size=1.2] coordinates {(128.0, 17.0703125)(256.0, 16.87890625)(512.0, 16.642578125)(1024.0, 16.5712890625)(2048.0, 16.51806640625)(4096.0, 16.927978515625)(8192.0, 17.1065673828125)(16384.0, 17.34454345703125)(32768.0, 17.460601806640625)(65536.0, 17.846389770507812)(131072.0, 17.87220001220703)(262144.0, 18.149646759033203)(524288.0, 18.49161720275879)(1048576.0, 18.636475563049316)};
      \addlegendentry{$\lambda = 12$};
      \addplot+[mark size=1.2] coordinates {(128.0, 18.1875)(256.0, 17.26953125)(512.0, 17.267578125)(1024.0, 17.279296875)(2048.0, 17.38671875)(4096.0, 17.4150390625)(8192.0, 17.7354736328125)(16384.0, 17.72430419921875)(32768.0, 18.020233154296875)(65536.0, 18.24615478515625)(131072.0, 18.382125854492188)(262144.0, 18.67232894897461)(524288.0, 19.00275230407715)(1048576.0, 19.27247714996338)};
      \addlegendentry{$\lambda = 13$};
      \addplot+[mark size=1.2] coordinates {(128.0, 18.9296875)(256.0, 17.88671875)(512.0, 17.966796875)(1024.0, 17.6376953125)(2048.0, 17.88330078125)(4096.0, 18.218017578125)(8192.0, 18.2486572265625)(16384.0, 18.31866455078125)(32768.0, 18.571990966796875)(65536.0, 18.712539672851562)(131072.0, 19.155906677246094)(262144.0, 19.288936614990234)(524288.0, 19.138330459594727)(1048576.0, 19.418910026550293)};
      \addlegendentry{$\lambda = 14$};
      \addplot+[mark size=1.2] coordinates {(128.0, 19.34375)(256.0, 18.51953125)(512.0, 18.37109375)(1024.0, 18.384765625)(2048.0, 18.4501953125)(4096.0, 18.684326171875)(8192.0, 18.5780029296875)(16384.0, 18.967041015625)(32768.0, 18.972503662109375)(65536.0, 19.40118408203125)(131072.0, 19.41193389892578)(262144.0, 19.725784301757812)(524288.0, 19.803573608398438)(1048576.0, 20.184317588806152)};
      \addlegendentry{$\lambda = 15$};
      \addplot+[mark size=1.2] coordinates {(128.0, 20.0078125)(256.0, 19.94140625)(512.0, 19.470703125)(1024.0, 18.9228515625)(2048.0, 19.22705078125)(4096.0, 19.058837890625)(8192.0, 19.4552001953125)(16384.0, 19.44244384765625)(32768.0, 19.886260986328125)(65536.0, 20.047378540039062)(131072.0, 20.08460235595703)(262144.0, 20.278934478759766)(524288.0, 20.496767044067383)(1048576.0, 20.47352695465088)};
      \addlegendentry{$\lambda = 16$};
      \addplot+[mark size=1.2] coordinates {(128.0, 21.2578125)(256.0, 20.32421875)(512.0, 19.591796875)(1024.0, 19.4248046875)(2048.0, 19.78125)(4096.0, 19.84326171875)(8192.0, 19.9344482421875)(16384.0, 20.32244873046875)(32768.0, 20.373260498046875)(65536.0, 20.5130615234375)(131072.0, 20.534584045410156)(262144.0, 20.76998519897461)(524288.0, 21.105127334594727)(1048576.0, 20.912473678588867)};
      \addlegendentry{$\lambda = 17$};
      \addplot+[mark size=1.2] coordinates {(128.0, 21.8046875)(256.0, 21.51953125)(512.0, 20.955078125)(1024.0, 20.2333984375)(2048.0, 20.71630859375)(4096.0, 20.628173828125)(8192.0, 20.6983642578125)(16384.0, 20.91583251953125)(32768.0, 20.973480224609375)(65536.0, 21.077835083007812)(131072.0, 21.205543518066406)(262144.0, 21.272350311279297)(524288.0, 21.56599235534668)(1048576.0, 21.552550315856934)};
      \addlegendentry{$\lambda = 18$};
      \addplot+[mark size=1.2] coordinates {(128.0, 22.5703125)(256.0, 22.12109375)(512.0, 21.04296875)(1024.0, 21.1162109375)(2048.0, 21.06005859375)(4096.0, 20.953125)(8192.0, 21.4215087890625)(16384.0, 21.47247314453125)(32768.0, 21.539703369140625)(65536.0, 21.726974487304688)(131072.0, 21.92816925048828)(262144.0, 21.97187042236328)(524288.0, 22.121212005615234)(1048576.0, 22.241454124450684)};
      \addlegendentry{$\lambda = 19$};
      \addplot+[mark size=1.2] coordinates {(128.0, 23.2890625)(256.0, 22.19140625)(512.0, 22.033203125)(1024.0, 21.9345703125)(2048.0, 22.28564453125)(4096.0, 21.723876953125)(8192.0, 22.0435791015625)(16384.0, 22.07525634765625)(32768.0, 22.174102783203125)(65536.0, 22.196060180664062)(131072.0, 22.59735870361328)(262144.0, 22.650074005126953)(524288.0, 22.5800724029541)(1048576.0, 22.83241367340088)};
      \addlegendentry{$\lambda = 20$};
      \addplot+[mark size=1.2] coordinates {(128.0, 13.359375)(256.0, 13.9921875)(512.0, 13.70703125)(1024.0, 13.6240234375)(2048.0, 13.91650390625)(4096.0, 14.199951171875)(8192.0, 14.564208984375)(16384.0, 14.94732666015625)(32768.0, 15.176849365234375)(65536.0, 15.294036865234375)(131072.0, 15.786491394042969)(262144.0, 15.834220886230469)(524288.0, 16.071090698242188)(1048576.0, 16.624646186828613)};
      \addlegendentry{$\lambda \le 2 \ln n$};
      \addplot+[mark size=1.2] coordinates {(128.0, 13.5703125)(256.0, 13.6328125)(512.0, 14.3828125)(1024.0, 15.5166015625)(2048.0, 18.39501953125)(4096.0, 20.803955078125)(8192.0, 21.651611328125)(16384.0, 25.49761962890625)(32768.0, 30.144775390625)(65536.0, 33.86186218261719)};
      \addlegendentry{$\lambda \le n$};
      \addplot+[mark size=1.2] coordinates {(128.0, 12.19921875)(256.0, 13.435546875)(512.0, 15.3486328125)(1024.0, 17.83740234375)(2048.0, 19.114013671875)(4096.0, 20.206787109375)(8192.0, 22.68341064453125)(16384.0, 24.479949951171875)(32768.0, 26.366470336914062)(65536.0, 27.592208862304688)(131072.0, 29.965496063232422)(262144.0, 31.83588409423828)(524288.0, 34.200706481933594)(1048576.0, 35.30829954147339)};
      \addlegendentry{(1+1) EA};
    \end{axis}
  \end{tikzpicture}
}
\newcommand{\iqrPlotOneMax}[3]{
  \begin{tikzpicture}[#3]
    \begin{axis}[enlargelimits=false, xmode=log, log basis x = 2, cycle list name = my custom, xlabel=Problem size, ylabel=Evaluations / problem size, width=#1, height=#2, legend pos=outer north east]
      \addplot+[mark size=1.2] coordinates {(16.0, 7.5625)(32.0, 9.46875)(64.0, 10.734375)(128.0, 14.1171875)(256.0, 14.73046875)(512.0, 17.841796875)(1024.0, 19.2451171875)(2048.0, 21.55126953125)(4096.0, 23.472412109375)(8192.0, 25.3607177734375)(16384.0, 28.58953857421875)(32768.0, 29.920013427734375)(65536.0, 32.18571472167969)(131072.0, 34.274559020996094)(262144.0, 35.907962799072266)(524288.0, 38.03854179382324)(1048576.0, 40.35618877410889)(2097152.0, 42.13673734664917)(4194304.0, 44.543511152267456)(8388608.0, 46.052051424980164)(1.6777216E7, 49.091919004917145)};
      \addlegendentry{$\lambda = 2$};
      \addplot+[mark size=1.2] coordinates {(16.0, 6.0625)(32.0, 7.71875)(64.0, 9.578125)(128.0, 10.7890625)(256.0, 12.3671875)(512.0, 12.880859375)(1024.0, 15.2294921875)(2048.0, 16.3173828125)(4096.0, 18.01416015625)(8192.0, 19.8443603515625)(16384.0, 21.072021484375)(32768.0, 22.263824462890625)(65536.0, 23.445388793945312)(131072.0, 25.55260467529297)(262144.0, 26.886123657226562)(524288.0, 27.74506378173828)(1048576.0, 29.230342864990234)(2097152.0, 31.040810585021973)(4194304.0, 32.053608894348145)(8388608.0, 33.467419385910034)(1.6777216E7, 35.51947569847107)};
      \addlegendentry{$\lambda = 3$};
      \addplot+[mark size=1.2] coordinates {(16.0, 6.5625)(32.0, 7.53125)(64.0, 8.140625)(128.0, 9.4453125)(256.0, 10.48828125)(512.0, 11.986328125)(1024.0, 12.7509765625)(2048.0, 14.02978515625)(4096.0, 15.172119140625)(8192.0, 16.3360595703125)(16384.0, 17.32867431640625)(32768.0, 18.116241455078125)(65536.0, 19.471389770507812)(131072.0, 20.67157745361328)(262144.0, 21.899311065673828)(524288.0, 23.214738845825195)(1048576.0, 23.96878147125244)(2097152.0, 25.276766300201416)(4194304.0, 26.155468225479126)(8388608.0, 27.11624777317047)(1.6777216E7, 28.50543361902237)};
      \addlegendentry{$\lambda = 4$};
      \addplot+[mark size=1.2] coordinates {(16.0, 6.9375)(32.0, 7.53125)(64.0, 8.53125)(128.0, 9.0703125)(256.0, 9.984375)(512.0, 10.998046875)(1024.0, 11.646484375)(2048.0, 12.79833984375)(4096.0, 13.81494140625)(8192.0, 14.380615234375)(16384.0, 15.7083740234375)(32768.0, 16.33197021484375)(65536.0, 16.936737060546875)(131072.0, 17.928436279296875)(262144.0, 19.08646011352539)(524288.0, 19.933378219604492)(1048576.0, 20.515385627746582)(2097152.0, 21.973851203918457)(4194304.0, 22.66824507713318)(8388608.0, 23.176756024360657)(1.6777216E7, 24.54081481695175)};
      \addlegendentry{$\lambda = 5$};
      \addplot+[mark size=1.2] coordinates {(16.0, 7.5625)(32.0, 7.53125)(64.0, 8.359375)(128.0, 9.1015625)(256.0, 10.03515625)(512.0, 10.947265625)(1024.0, 11.3505859375)(2048.0, 12.11767578125)(4096.0, 13.299560546875)(8192.0, 13.6539306640625)(16384.0, 14.34381103515625)(32768.0, 15.219207763671875)(65536.0, 16.022018432617188)(131072.0, 16.597328186035156)(262144.0, 17.170833587646484)(524288.0, 18.051992416381836)(1048576.0, 18.81623935699463)(2097152.0, 19.83958387374878)(4194304.0, 20.418909311294556)(8388608.0, 21.060561537742615)(1.6777216E7, 21.950752675533295)};
      \addlegendentry{$\lambda = 6$};
      \addplot+[mark size=1.2] coordinates {(16.0, 7.9375)(32.0, 7.90625)(64.0, 8.546875)(128.0, 9.4140625)(256.0, 10.01171875)(512.0, 10.392578125)(1024.0, 10.9931640625)(2048.0, 11.70361328125)(4096.0, 12.250244140625)(8192.0, 13.1370849609375)(16384.0, 13.7689208984375)(32768.0, 14.277313232421875)(65536.0, 15.152511596679688)(131072.0, 15.650726318359375)(262144.0, 16.51156997680664)(524288.0, 17.170368194580078)(1048576.0, 17.44979953765869)(2097152.0, 18.26900005340576)(4194304.0, 18.786408185958862)(8388608.0, 19.809128403663635)(1.6777216E7, 20.09312641620636)};
      \addlegendentry{$\lambda = 7$};
      \addplot+[mark size=1.2] coordinates {(16.0, 8.0625)(32.0, 9.03125)(64.0, 8.765625)(128.0, 9.7578125)(256.0, 10.12890625)(512.0, 10.611328125)(1024.0, 11.2041015625)(2048.0, 11.74658203125)(4096.0, 12.496337890625)(8192.0, 12.8760986328125)(16384.0, 13.55035400390625)(32768.0, 13.933380126953125)(65536.0, 14.612075805664062)(131072.0, 15.162483215332031)(262144.0, 16.011539459228516)(524288.0, 16.211732864379883)(1048576.0, 16.929436683654785)(2097152.0, 17.466515064239502)(4194304.0, 17.99393105506897)(8388608.0, 18.47956955432892)(1.6777216E7, 19.305882513523102)};
      \addlegendentry{$\lambda = 8$};
      \addplot+[mark size=1.2] coordinates {(16.0, 9.0625)(32.0, 9.59375)(64.0, 9.4375)(128.0, 10.1328125)(256.0, 10.12890625)(512.0, 10.794921875)(1024.0, 11.400390625)(2048.0, 11.95361328125)(4096.0, 12.467529296875)(8192.0, 12.973876953125)(16384.0, 13.374267578125)(32768.0, 14.04522705078125)(65536.0, 14.3524169921875)(131072.0, 14.900146484375)(262144.0, 15.421649932861328)(524288.0, 15.777811050415039)(1048576.0, 16.311711311340332)(2097152.0, 17.000060081481934)(4194304.0, 17.485162019729614)(8388608.0, 17.948797821998596)(1.6777216E7, 18.309204578399658)};
      \addlegendentry{$\lambda = 9$};
      \addplot+[mark size=1.2] coordinates {(16.0, 10.0625)(32.0, 10.03125)(64.0, 10.328125)(128.0, 10.2421875)(256.0, 10.74609375)(512.0, 11.193359375)(1024.0, 11.5439453125)(2048.0, 11.87548828125)(4096.0, 12.365966796875)(8192.0, 12.8724365234375)(16384.0, 13.64385986328125)(32768.0, 13.958465576171875)(65536.0, 14.350601196289062)(131072.0, 14.965522766113281)(262144.0, 15.317272186279297)(524288.0, 15.756799697875977)(1048576.0, 16.119290351867676)(2097152.0, 16.65387201309204)(4194304.0, 17.035730123519897)(8388608.0, 17.655608773231506)(1.6777216E7, 18.109706699848175)};
      \addlegendentry{$\lambda = 10$};
      \addplot+[mark size=1.2] coordinates {(16.0, 12.4375)(32.0, 10.6875)(64.0, 10.671875)(128.0, 10.8359375)(256.0, 11.1328125)(512.0, 11.904296875)(1024.0, 11.9248046875)(2048.0, 12.39697265625)(4096.0, 12.85595703125)(8192.0, 13.19287109375)(16384.0, 13.63055419921875)(32768.0, 13.996429443359375)(65536.0, 14.50213623046875)(131072.0, 14.75750732421875)(262144.0, 15.438037872314453)(524288.0, 15.636892318725586)(1048576.0, 16.04175853729248)(2097152.0, 16.632872104644775)(4194304.0, 16.919066190719604)(8388608.0, 17.329612493515015)(1.6777216E7, 17.74625313282013)};
      \addlegendentry{$\lambda = 11$};
      \addplot+[mark size=1.2] coordinates {(16.0, 12.0625)(32.0, 11.28125)(64.0, 11.265625)(128.0, 11.2578125)(256.0, 11.72265625)(512.0, 12.048828125)(1024.0, 12.3876953125)(2048.0, 12.77392578125)(4096.0, 13.002197265625)(8192.0, 13.3609619140625)(16384.0, 13.85235595703125)(32768.0, 14.249298095703125)(65536.0, 14.612197875976562)(131072.0, 14.871559143066406)(262144.0, 15.609561920166016)(524288.0, 15.571451187133789)(1048576.0, 16.03372287750244)(2097152.0, 16.386085987091064)(4194304.0, 16.80875039100647)(8388608.0, 17.32966673374176)(1.6777216E7, 17.438276827335358)};
      \addlegendentry{$\lambda = 12$};
      \addplot+[mark size=1.2] coordinates {(16.0, 14.6875)(32.0, 11.8125)(64.0, 12.0)(128.0, 11.5859375)(256.0, 12.140625)(512.0, 12.443359375)(1024.0, 12.5439453125)(2048.0, 13.08935546875)(4096.0, 13.34619140625)(8192.0, 13.7808837890625)(16384.0, 14.21405029296875)(32768.0, 14.490936279296875)(65536.0, 14.770217895507812)(131072.0, 15.107925415039062)(262144.0, 15.514617919921875)(524288.0, 15.835941314697266)(1048576.0, 16.183486938476562)(2097152.0, 16.632564544677734)(4194304.0, 16.851756811141968)(8388608.0, 17.26349115371704)(1.6777216E7, 17.641279697418213)};
      \addlegendentry{$\lambda = 13$};
      \addplot+[mark size=1.2] coordinates {(16.0, 14.0625)(32.0, 13.15625)(64.0, 12.265625)(128.0, 12.4765625)(256.0, 12.69140625)(512.0, 12.798828125)(1024.0, 13.0849609375)(2048.0, 13.63134765625)(4096.0, 13.894287109375)(8192.0, 14.0001220703125)(16384.0, 14.48114013671875)(32768.0, 14.790863037109375)(65536.0, 15.153793334960938)(131072.0, 15.334617614746094)(262144.0, 15.70151138305664)(524288.0, 15.926107406616211)(1048576.0, 16.203883171081543)(2097152.0, 16.599198818206787)(4194304.0, 16.83663058280945)(8388608.0, 17.340455412864685)(1.6777216E7, 17.604559361934662)};
      \addlegendentry{$\lambda = 14$};
      \addplot+[mark size=1.2] coordinates {(16.0, 15.0625)(32.0, 13.15625)(64.0, 13.140625)(128.0, 12.6640625)(256.0, 12.89453125)(512.0, 13.361328125)(1024.0, 13.6826171875)(2048.0, 13.7333984375)(4096.0, 14.16162109375)(8192.0, 14.5386962890625)(16384.0, 14.751953125)(32768.0, 15.112640380859375)(65536.0, 15.306259155273438)(131072.0, 15.730140686035156)(262144.0, 15.854248046875)(524288.0, 16.31023597717285)(1048576.0, 16.590901374816895)(2097152.0, 16.851590156555176)(4194304.0, 17.224738836288452)(8388608.0, 17.51703155040741)(1.6777216E7, 17.803229451179504)};
      \addlegendentry{$\lambda = 15$};
      \addplot+[mark size=1.2] coordinates {(16.0, 18.0625)(32.0, 15.03125)(64.0, 13.515625)(128.0, 13.5078125)(256.0, 13.50390625)(512.0, 13.845703125)(1024.0, 14.1103515625)(2048.0, 14.17236328125)(4096.0, 14.703369140625)(8192.0, 14.8419189453125)(16384.0, 15.21099853515625)(32768.0, 15.405303955078125)(65536.0, 15.671401977539062)(131072.0, 15.869392395019531)(262144.0, 16.276493072509766)(524288.0, 16.528047561645508)(1048576.0, 16.748963356018066)(2097152.0, 17.123711109161377)(4194304.0, 17.35989785194397)(8388608.0, 17.54147160053253)(1.6777216E7, 17.956582129001617)};
      \addlegendentry{$\lambda = 16$};
      \addplot+[mark size=1.2] coordinates {(16.0, 18.9375)(32.0, 15.96875)(64.0, 14.359375)(128.0, 14.0859375)(256.0, 13.94921875)(512.0, 14.478515625)(1024.0, 14.4609375)(2048.0, 14.82568359375)(4096.0, 14.94580078125)(8192.0, 15.37939453125)(16384.0, 15.56195068359375)(32768.0, 15.913665771484375)(65536.0, 16.02337646484375)(131072.0, 16.271484375)(262144.0, 16.559608459472656)(524288.0, 16.808176040649414)(1048576.0, 17.11617946624756)(2097152.0, 17.3337664604187)(4194304.0, 17.732309103012085)(8388608.0, 17.970228910446167)(1.6777216E7, 18.193804800510406)};
      \addlegendentry{$\lambda = 17$};
      \addplot+[mark size=1.2] coordinates {(16.0, 19.5625)(32.0, 15.78125)(64.0, 14.640625)(128.0, 14.3515625)(256.0, 14.48828125)(512.0, 14.837890625)(1024.0, 14.9248046875)(2048.0, 15.34619140625)(4096.0, 15.565673828125)(8192.0, 15.7962646484375)(16384.0, 15.96429443359375)(32768.0, 16.174652099609375)(65536.0, 16.397933959960938)(131072.0, 16.613258361816406)(262144.0, 16.819385528564453)(524288.0, 17.20610237121582)(1048576.0, 17.45034122467041)(2097152.0, 17.66011095046997)(4194304.0, 17.87845540046692)(8388608.0, 18.11461842060089)(1.6777216E7, 18.42073565721512)};
      \addlegendentry{$\lambda = 18$};
      \addplot+[mark size=1.2] coordinates {(16.0, 18.875)(32.0, 17.84375)(64.0, 15.453125)(128.0, 15.1484375)(256.0, 14.99609375)(512.0, 15.291015625)(1024.0, 15.4013671875)(2048.0, 15.67919921875)(4096.0, 15.947998046875)(8192.0, 16.12646484375)(16384.0, 16.259765625)(32768.0, 16.625030517578125)(65536.0, 16.747650146484375)(131072.0, 17.0501708984375)(262144.0, 17.341533660888672)(524288.0, 17.541683197021484)(1048576.0, 17.728642463684082)(2097152.0, 17.944448947906494)(4194304.0, 18.205464601516724)(8388608.0, 18.45515775680542)(1.6777216E7, 18.692143321037292)};
      \addlegendentry{$\lambda = 19$};
      \addplot+[mark size=1.2] coordinates {(16.0, 21.5625)(32.0, 16.90625)(64.0, 15.640625)(128.0, 15.9453125)(256.0, 15.47265625)(512.0, 15.861328125)(1024.0, 15.7626953125)(2048.0, 16.19189453125)(4096.0, 16.313720703125)(8192.0, 16.5894775390625)(16384.0, 16.76885986328125)(32768.0, 16.970245361328125)(65536.0, 17.163406372070312)(131072.0, 17.470863342285156)(262144.0, 17.651981353759766)(524288.0, 17.894479751586914)(1048576.0, 18.084927558898926)(2097152.0, 18.348083972930908)(4194304.0, 18.556137323379517)(8388608.0, 18.75535261631012)(1.6777216E7, 19.013918697834015)};
      \addlegendentry{$\lambda = 20$};
      \addplot+[mark size=1.2] coordinates {(16.0, 7.625)(32.0, 8.625)(64.0, 8.890625)(128.0, 9.5234375)(256.0, 9.7265625)(512.0, 10.365234375)(1024.0, 10.6005859375)(2048.0, 10.751953125)(4096.0, 10.93701171875)(8192.0, 11.101806640625)(16384.0, 11.21221923828125)(32768.0, 11.4393310546875)(65536.0, 11.394882202148438)(131072.0, 11.566932678222656)(262144.0, 11.69482421875)(524288.0, 11.776792526245117)(1048576.0, 11.836039543151855)(2097152.0, 11.928204536437988)(4194304.0, 12.038527488708496)(8388608.0, 12.120957493782043)(1.6777216E7, 12.07781457901001)};
      \addlegendentry{$\lambda \le 2 \ln n$};
      \addplot+[mark size=1.2] coordinates {(16.0, 6.9375)(32.0, 8.6875)(64.0, 9.40625)(128.0, 9.4609375)(256.0, 9.7265625)(512.0, 10.09375)(1024.0, 10.373046875)(2048.0, 10.3720703125)(4096.0, 10.48095703125)(8192.0, 10.6064453125)(16384.0, 10.61883544921875)(32768.0, 10.656280517578125)(65536.0, 10.6962890625)(131072.0, 10.728523254394531)(262144.0, 10.740795135498047)(524288.0, 10.743402481079102)(1048576.0, 10.749378204345703)(2097152.0, 10.76009225845337)(4194304.0, 10.773239850997925)(8388608.0, 10.775318145751953)(1.6777216E7, 10.778870224952698)};
      \addlegendentry{$\lambda \le n$};
      \addplot+[mark size=1.2] coordinates {(16.0, 5.4375)(32.0, 6.640625)(64.0, 8.96875)(128.0, 10.9375)(256.0, 12.720703125)(512.0, 14.248046875)(1024.0, 16.830078125)(2048.0, 18.3720703125)(4096.0, 20.571044921875)(8192.0, 22.3193359375)(16384.0, 23.89788818359375)(32768.0, 25.918731689453125)(65536.0, 27.663597106933594)(131072.0, 30.043193817138672)(262144.0, 31.92898941040039)(524288.0, 33.06271839141846)(1048576.0, 35.72658729553223)(2097152.0, 37.06888699531555)(4194304.0, 38.919840574264526)(8388608.0, 41.35184806585312)(1.6777216E7, 43.0642252266407)};
      \addlegendentry{(1+1) EA};
    \end{axis}
  \end{tikzpicture}
}

%% file: parts/0-abstract.tex
The $\onell$ genetic algorithm, first proposed at GECCO 2013, showed a surprisingly good performance on so me optimization problems. The theoretical analysis so far was restricted to the \textsc{OneMax} test function, where this GA profited from the perfect fitness-distance correlation. In this work, we conduct a rigorous runtime analysis of this GA on random 3-SAT instances in the planted solution model having at least logarithmic average degree, which are known to have a weaker fitness distance correlation.

We prove that this GA with fixed not too large population size again obtains runtimes better than $\Theta(n \log n)$, which is a lower bound for most evolutionary algorithms on pseudo-Boolean problems with unique optimum. However, the self-adjusting version of the GA risks reaching population sizes at which the intermediate selection of the GA, due to the weaker fitness-distance correlation, is not able to distinguish a profitable offspring from others. We show that this problem can be overcome by equipping the self-adjusting GA with an upper limit for the population size. Apart from sparse instances, this limit can be chosen in a way that the asymptotic performance does not worsen compared to the idealistic \textsc{OneMax} case. Overall, this work shows that the $\onell$ GA can provably have a good performance on combinatorial search and optimization problems also in the presence of a weaker fitness-distance correlation.

%% file: parts/1-introduction.tex
\section{Introduction}

The $\onell$ genetic algorithm (GA) was fairly recently introduced by Doerr, Doerr, and Ebel~\cite{learning-from-black-box}. It builds on the simple idea to first generate from a single parent individual several offspring via standard-bit mutation with higher-than-usual mutation rate, to select the best of these, and to perform a biased crossover with the parent to reduce the destructive effects of the high mutation rate. This use of crossover with the parent as repair mechanism is novel in evolutionary discrete optimization. 

The so far moderate number of results on this GA show that it has some remarkable properties. A mathematical runtime analysis on the \onemax test function class~\cite{doerr-doerr-lambda-lambda-self-adjustment,doerr-doerr-lambda-lambda-fixed-tight,learning-from-black-box-thcs,Doerr16} shows that the right parameter setting leads to an optimization time of slightly better than $O(n \sqrt{\log n})$. This is remarkable as all previous runtime analyses of evolutionary algorithms on the \onemax test function class showed that these algorithms needed at least $\Omega(n \log n)$ fitness evaluations. The result is remarkable also in that it is the first time that crossover was rigorously shown to give an asymptotic runtime improvement for a simple test function. A third noteworthy property of this GA is that a simple self-adjusting choice of the offspring population size $\lambda$ inspired by the $1/5$-th rule from continuous optimization could further improve the runtime to $\Theta(n)$. Again, this is the first time that a $1/5$-th rule type dynamic parameter choice could be proven useful in discrete evolutionary optimization. These mathematical analyses are complemented by an experimental investigation on the \onemax test function, on linear functions with random weights, and on royal road functions. Further, Goldman and Punch~\cite{goldman-punch-ppp} in the analysis of their parameter-less population pyramid algorithm also used the $\onell$ GA as comparison, and it performed well on random MAX-SAT instances.

It is clear that the intermediate selection of the best mutation offspring becomes most effective if there is a strong fitness-distance correlation. The \onemax function, by definition, has a perfect fitness-distance correlation. So the question has to be asked to what degree the positive results for \onemax remain true in optimization problems with a weaker fitness distance correlation. To study this question, we analyze the performance of the \ga on random satisfiable 3-SAT instances in the planted solution model. For these, the fitness-distance correlation can be scaled via the instance density. In works~\cite{SuttonN14,doerr-neumann-sutton-oneplusone-cnf}, the analysis of the performance of the \oea on these instances has shown that when the clause-variable ratio is $m/n = \Omega(n)$, then the following strong fitness-distance correlation holds apart from an exponentially small failure probability: For any two search points $x, y$ which are least half as good as a random search point and such that they differ in one bit, the fitness and the distance are perfectly correlated in that the one closer to the optimum has a fitness larger by $\Theta(m/n)$. From this, one could easily derive an $O(n \log n)$ optimization time of the \mbox{\oea}. However, when the clause-variable ratio is only logarithmic, this strong fitness-distance correlation is far from being satisfied. Therefore, only the much weaker condition could be shown that each pair $(x,y)$ as above shows a $\Theta(m/n)$ fitness advantage of the closer search point with probability $1-n^3$. This was enough to show that with high probability (over the joint probability space of instance and algorithm) the \oea finds the optimum in $O(n \log n)$ iterations.

\textbf{Our results:} We conduct a rigorous runtime analysis of the \ga (with mutation rate $p = \lambda/n$ and crossover bias $c = 1/\lambda$ linked to the population size $\lambda$ as recommended in~\cite{Doerr16}), on the same type of random 3-SAT instance (see Section~\ref{sec:preliminaries} for the details) as regarded in~\cite{SuttonN14,doerr-neumann-sutton-oneplusone-cnf}. We observe that the weaker fitness-distance correlation of low-density instance indeed poses a problem for the \ga when the population size (and thus the mutation rate) is high. In this case, the mutation offspring are distant enough so that the weak fitness-distance correlation prevents the \ga to detect an individual closer to the optimum than the typical offspring. For the \ga with static value of $\lambda$, our experiments and informal considerations suggest that in this case the \ga, for large enough $n$, reverts to the behavior similar to the \oea, however, the constant in $O(n \log n)$ is proportional to $1 / \lambda$, just like it happens for \onemax. Things are worse when the self-adaptive version of the \ga is used. In this case, the low probability to find an improving solution lets the value of $\lambda$ further increase as governed by the $1/5$-th rule parameter adaptation. Consequently, the probability to select a profitable individual out of the mutation offspring population further decreases and this negative effect becomes stronger. The result is that $\lambda$ quickly reaches the allowed maximum of $n$ and the performance approaches the one of the \oea with the cost of one iteration $n$ times higher than usual. 

On the positive side, we make precise that these negative effect can be overcome in most cases by an appropriate choice of $\lambda$. We show that when $\lambda$ is asymptotically smaller that the fourth root of the density, then the probability to find an improving solution is asymptotically the same as for the optimization of \onemax. Consequently, when the density is $\omega(\log^2 n)$, then we can still use $\lambda = \sqrt{\log n}$ and obtain an expected optimization time (number of fitness evaluations) of $O(n \sqrt{\log n})$. Note that in this work, we do not try to achieve the later improvement of this runtime guarantee~\cite{doerr-doerr-lambda-lambda-fixed-tight} by a factor $\Theta(\sqrt{\log\log\log n / \log\log n})$ though we are optimistic that such a guarantee can be shown with mildly more effort. For a logarithmic instance density, the smallest regarded here and in~\cite{SuttonN14,doerr-neumann-sutton-oneplusone-cnf}, for any $\eps > 0$ with a choice of $\lambda = \log^{0.25-\eps} n$ we still obtain a runtime of $O(n \log^{0.75+\eps} n)$ and beat the $O(n \log n)$ performance the \oea has on these instances.

For the self-adjusting version, where a typical run of the \ga on \onemax uses $\lambda$-values of up to $\sqrt n$, we show that adding an upper limit up to which the value of $\lambda$ can at most grow, overcomes this difficulties sketched above. The runtime increase incurred by such a limit is again manageable. If the upper limit is $\lambdabound$ and it depends on the instance density as $\lambdabound = o((m/n)^4)$, then the runtime of the self-adjusting \ga is
\begin{equation*}
    O\left(n \cdot \max\left\{1, \frac{\log n}{\lambdabound}\right\}\right). 
\end{equation*}

Hence already for densities asymptotically larger than $\log^4 n$, we obtain the linear runtime that is valid for the ideal \onemax fitness landscape for densities at least logarithmic.

\textbf{Techniques employed:} Our main result that $\lambda = o((m/n)^{1/4})$ suffices for the well-functioning of the \ga on our random 3-SAT instances of density $m/n$ is based on a different fitness-distance correlation result than those used in~\cite{SuttonN14,doerr-neumann-sutton-oneplusone-cnf}. Whereas the latter require that with good probability all neighbors of a given solution (not excessively far from the optimum) have a fitness advantage or disadvantage of order $m/n$ (depending on whether they are closer to the optimum or further away), we require this condition only for a certain fraction of the neighbors. This relaxation will not be a problem since it only reduces the probability that the \ga finds a certain improvement by a constant factor (being an elitist algorithm, we do not need to care about fitness losses). On the positive side, this relaxation (i) allows us to extend the fitness-distance correlation requirement to all vertices in the $\lambda$-neighborhood instead of only all direct neighbors and (ii) gives us that this correlation property with probability $1-\exp(-\Omega(n))$ holds for all vertices (not excessively far from the optimum). Consequently, the performance results we show hold for all but an exponentially small fraction of the input instances. This fitness-distance correlation result also implies that the result of~\cite{doerr-neumann-sutton-oneplusone-cnf} for logarithmic densities holds in the same strong version as the one for linear densities, namely that on all but an exponentially small fraction of the input instances the expected runtime is $O(n \log n)$. 

To prove our fitness-distance correlation result, we use McDiarmid's bounded differences version~\cite{mcdiarmid-inequality} of the Azuma martigale concentration inequality in a novel way. To reduce the maximum influence of the independent basic random variables on the discrete quantity of interest, we replace this quantity by a larger continuous function in a way that the influence of each basic random variable is significantly reduced. We are not aware of this type of argument being used before, either in evolutionary computation or randomized algorithms in general.

%% file: parts/2-preliminaries.tex
\section{Preliminaries}\label{sec:preliminaries}

In this section, we define the notation, algorithms, and problems we regard in this work. Our notation is standard. We write $[a..b]$ to denote the set of all integers in the real interval $[a;b]$. We write $[r]$ to denote the integer closest to the real number $r$, rounding up in case of ties.	

\subsection{The $k$-CNF SAT Problem}

Consider a set $V$ of $n$ Boolean variables $V = \{x_1, x_2, \ldots, x_n\}$. A \emph{clause} $C$ over $V$ is the logical disjunction of exactly $k$ literals, $C = l_{1} \lor l_{2} \ldots \lor l_{k}$, and each literal $l_{i}$ is either a variable $x_j$ or its negation $\lnot x_j$. A \emph{$k$-CNF formula} $F$ is a logical conjunction of exactly $m$ clauses $F = C_1 \land C_2 \land \ldots \land C_m$. A $k$-CNF formula $F$ is \emph{satisfiable} if and only if there is an assignment of truth values to the variables such that every clause contains at least one true literal, i.e., the whole expression $F$ evaluates to true. We shall only regard the case $k=3$, but \revisepar{we have no doubts}{Reviewer 1 asks to check this statement, especially for $k = \omega(1)$.} that the main claims are true for the general case as well.

We consider random 3-CNF formulas consisting of $m$ clauses of length $k = 3$ over the $n$ variables in $V$. We take the usual assumption that each clause consists of distinct variables. This assumption is very natural since any clause of length 3 that contains repeating variables can be immediately reduced to an equivalent clause of length 2 or, alternatively, to a tautology. However, we explicitly allow repeated clauses in $F$.

Let $\Omega_{n, m}$ be the finite set of all 3-CNF formulas over $n$ variables and $m$ clauses. We work in the so-called \emph{planted solution} model. Hence there is a target assignment $x^*$ and $F$ is a formula chosen uniformly at random among all formulas in $\Omega_{n,m}$ which are satisfied by $x^*$. We refer to~\cite{SuttonN14,doerr-neumann-sutton-oneplusone-cnf} for a justification of this model and a discussion how it relates to other random satisfiability problems.

We shall, without loss of generality, assume that the planted solution is $x^* = (1,\dots,1)$. This is justified by the fact that we only regard unbiased algorithms, that is, algorithms that treat bit positions and the bit values $0$ and $1$ symmetrically. Hence these algorithms cannot profit from ``knowing'' the optimal solution. A  random formula in this model can be constructed by $m$ times (with replacement) choosing a random clause satisfied by $x^*$, that is, a random clause among all clauses containing at least one positive literal. Note that such a random formula may have other satisfying assignment than $x^*$. Nevertheless, we denote the structural distance, which is the Hamming distance here, of a solution $x$ to the planted solution $x^*$ by $d(x) =  |\{i : x_i = 0\}|$. When talking about fitness-distance correlation and related concepts, we shall always refer to this distance.

\subsection{3-CNF and Evolutionary Algorithms}

An assignment of true/false values to a set of $n$ Boolean variables can be represented by a bit string $x \in \{0,1\}^n$ such that $x_i= 1$ if and only if the $i$-th variable is having the value true. For a length-$m$ formula $F$ on $n$ variables, we define the fitness function $f = f_F : \{0,1\}^n \to [0..m]$ via $f(x) =  |\{C \in F \mid C \text{ is satisfied by } x\}|$, the number of clauses satisfied by the assignment represented by $x$. If $F$ is satisfiable, the task of finding a satisfying assignment reduces to the task of maximizing $f$.

In~\cite{doerr-neumann-sutton-oneplusone-cnf}, it is proven that when $m/n > c \ln n$ for sufficiently large constant $c$, the runtime of the $(1+1)$ EA is $O(n \log n)$ with probability polynomially close to one. One of the key concepts of the proof is the \emph{fitness-distance correlation}. In the case of logarithmic density, this concept can be formulated as follows:
\begin{lemma}[Lemma 4 from \cite{doerr-neumann-sutton-oneplusone-cnf}]\label{lemma-one-bit-flip-is-ok}
Let $0 < \eps < \tfrac 12$. Assume that $m/n > c \ln n$ for sufficiently large constant $c$. Then there exist two constants $c_1$ and $c_2$
such that, for any two solutions $x_1$ and $x_2$ such that
\begin{itemize}
    \item they are different in exactly one bit;
    \item this bit is set to $1$ in $x_2$;
    \item the structural distance $d(x_1)$ from $x_1$ to the planted solution is at most $(1/2 + \varepsilon) n$;
\end{itemize}
we have $c_1 m/n \le f(x_2) - f(x_1) \le c_2 m/n$ with probability at least $1 - n^{-3}$.
\end{lemma}

\subsection{The {$\onell$} Genetic Algorithm}

The $\onell$ Genetic Algorithm, or the $\onell$~GA for short, was proposed by Doerr, Doerr and Ebel in~\cite{learning-from-black-box-thcs}. Its main working principles are (i) to use mutation with a higher-than-usual mutation rate to speed up exploration and (ii) crossover with the parent to diminish the destructive effects of this mutation. Two versions of the algorithm were proposed, one with static parameters and one with a self-adjusting parameter choice. 

\begin{algorithm}[t]
\caption{: The $\onell$~GA with fixed integer population size $\lambda$}\label{algo:ll-fixed}
\begin{algorithmic}[1]
\State{$x \gets \Call{UniformRandom}{\{0, 1\}^n}$} \Comment{Initialization}
\For{$t \gets 1, 2, 3, \ldots$} \Comment{Optimization}
    \State{$p \gets \lambda / n$}\Comment{Mutation probability}
    \State{$c \gets 1 / \lambda$}\Comment{Crossover probability}
    \State{$\ell \sim \mathcal{B}(n, p)$}\Comment{Mutation strength}
    \For{$i \in [1..\lambda]$} \Comment{Phase 1: Mutation}
        \State{$x^{(i)} \gets \Call{Mutate}{x, \ell}$}
    \EndFor
    \State{$x' \gets \Call{UniformRandom}{\{x^{(j)} \mid f(x^{(j)}) = \max\{f(x^{(i)})\}\}}$}
    \For{$i \in [1..\lambda]$} \Comment{Phase 2: Crossover}
        \State{$y^{(i)} \gets \Call{Crossover}{x, x', c}$}
    \EndFor
    \State{$y \gets \Call{UniformRandom}{\{y^{(j)} \mid f(y^{(j)}) = \max\{f(y^{(i)})\}\}}$}
    \If{$f(y) \ge f(x)$} \Comment{Selection}
        \State{$x \gets y$}
    \EndIf
\EndFor
\end{algorithmic}
\end{algorithm}

The fixed-parameter version is outlined in Algorithm~\ref{algo:ll-fixed}. It uses the following two variation operators.
\begin{itemize}
    \item $\ell$-bit mutation: The unary mutation operator $\Call{Mutate}{x, \ell}$ creates from $x \in \{0,1\}^n$ a new bit string $y$ by flipping exactly $\ell$ bits chosen randomly without replacement.
    \item biased uniform crossover: The binary crossover operator $\Call{Crossover}{x, x', c}$ with \emph{crossover bias} $c \in [0,1]$ constructs a new bit string $y$ from two given bit strings $x$ and $x'$ by choosing for each $i \in [1..n]$ the second argument's value ($y_i = x'_i$) with probability $c$ and setting $y_i = x_i$ otherwise.
\end{itemize}

The \ga has three parameters, the mutation rate $p$, the crossover bias $c$, and the offspring population size $\lambda$. After randomly initializing the one-element parent population $\{x\}$, in each iteration the following steps are performed:
\begin{itemize}
    \item In the \emph{mutation phase}, $\lambda$ offspring are sampled from the parent $x$ by applying $\lambda$ times independently the mutation operator $\Call{Mutate}{x, \ell}$, where the step size $\ell$ is chosen at random from the binomial distribution $\mathcal{B}(n, p)$. Consequently, each offspring has the distribution of standard bit mutation with mutation rate $p$, but all offspring have the same Hamming distance from the parent.
    \item In an \emph{intermediate selection} step, the mutation offspring with maximal fitness, called \emph{mutation winner} and denoted by $x'$, is determined (breaking ties randomly).
    \item In the \emph{crossover phase}, $\lambda$ offspring are created from $x$ and $x'$ using via the biased uniform crossover $\Call{Crossover}{x, x', c}$.
    \item \emph{Elitist selection}. The best of the crossover offspring (breaking ties randomly and ignoring individuals equal to $x$) replaces $x$ if its fitness is at least as large as the fitness of $x$. 
\end{itemize}

Throughout this paper we use the mutation rate $p = \lambda / n$ and the crossover bias $c = 1 / \lambda$ as recommended and justified  in~\cite[Sections 2 and 3]{learning-from-black-box-thcs} and~\cite[Section~6]{Doerr16}.

For this \ga, a first runtime analysis~\cite{learning-from-black-box-thcs} was conducted on the \onemax test function \[\onemax : \{0,1\}^n \to \Z; x \mapsto \sum_{i=1}^n x_i.\] It was shown that for arbitrary $\lambda$, possibly being a function of $n$, the expected optimization time (number of fitness evaluations until the optimum is evaluated for the first time) is $O(\max\{n\lambda, \frac{n \log n}{\lambda}\})$. This expression is minimized for $\lambda = \Theta(\sqrt{\log n})$, giving an upper bound of $O(n \sqrt{\log n})$. The analysis, on which we will build on in this work, uses the fitness level method~\cite{Wegener02EvOpt}. Roughly speaking, the arguments are that in an iteration starting with an individual $x$ with fitness distances $d=d(x)=n-\onemax(x)$ (i) with probability $\Omega(\min\{1,d\lambda^2/n\})$ the mutation winner is less than $\ell$ fitness levels worse than the parent, and that (ii) in this case with constant probability the crossover winner is better than the parent. Consequently, the expected number of iterations needed to gain an improvement from $x$ is $O(\max\{1,n/d\lambda^2\})$. Summing over all $d$ and noting that one iteration uses $2\lambda$ fitness evaluations gives the claim.

This result is interesting in that this is the first time that crossover was proven to bring an asymptotic speed-up for a simple fitness landscape like \onemax. Previously, a constant improvement was shown to be possible for \onemax using crossover~\cite{sudholt-crossover-speeds-up}. Note that all mutation-based algorithms that treat bit positions and bit values symmetrically need at least $\Omega(n \log n)$ fitness evaluations, as this is the unary unbiased black-box complexity of \onemax~\cite{LehreW12}. 

The \ga, which in principle nothing more than a \oea with a complicated mutation operator, in experiments showed a performance superior to the one of the classic \oea on \onemax (showing also that the constants hidden in the asymptotic notation are small), on linear functions and on royal road functions~\cite{learning-from-black-box-thcs} as well as on maximum satisfiability instances~\cite{goldman-punch-ppp}.

Subsequently, the runtime analysis on \onemax was improved~\cite{doerr-doerr-lambda-lambda-fixed-tight} and the tight bound of $\Theta(\max\{n \log n / \lambda, n \lambda \log \log \lambda / \log \lambda\})$ was shown for all values of $\lambda \le n$. This is minimized to $\Theta(n \sqrt{\log n \log \log \log n / \log \log n})$ when setting $\lambda = \Theta(\sqrt{\log n \log \log n / \log \log \log n})$.

Already in~\cite{learning-from-black-box-thcs}, it was observed that a dynamic choice of the parameter $\lambda$ can reduce the runtime to linear. For this, a fitness dependent choice of $\lambda = \lceil \frac{n}{n-\onemax(x)} \rceil$ suffices. This seems to be the first time that a super-constant speed-up was provably obtained by a dynamic parameter choice (see~\cite{BadkobehLS15} for a result showing 
that also the \lea can profit from a dynamic choice offspring population size when optimizing \onemax). Since a fitness-dependent choice as above is unlikely to be guessed by an algorithm user, also a self-adjusting variant setting $\lambda$ success-based according to a \mbox{$1/5$-th} rule was proposed in~\cite{learning-from-black-box-thcs}. That this indeed closely tracks the optimal value of 
$\lambda$ and gives a runtime of $O(n)$ was later shown in~\cite{doerr-doerr-lambda-lambda-self-adjustment}. The self-adjusting version of the $\onell$~GA is described in Algorithm~\ref{algo:ll-adjusting}.

\begin{algorithm}[t]
\caption{: The $\onell$~GA with self-adjusting $\lambda \le \lambdabound$}\label{algo:ll-adjusting}
\begin{algorithmic}[1]
\State{$F \gets \text{constant} \in (1; 2)$} \Comment{Update strength}
\State{$U \gets 5$} \Comment{The 5 from the ``1/5-th rule''}
\State{$x \gets \Call{UniformRandom}{\{0, 1\}^n}$}
\For{$t \gets 1, 2, 3, \ldots$}
    \State{$p \gets \lambda / n$, $c \gets 1 / \lambda$, $\lambda' \gets [\lambda]$, $\ell \sim \mathcal{B}(n, p)$}
    \For{$i \in [1..\lambda']$} \Comment{Phase 1: Mutation}
        \State{$x^{(i)} \gets \Call{Mutate}{x, \ell}$}
    \EndFor
    \State{$x' \gets \Call{UniformRandom}{\{x^{(j)} \mid f(x^{(j)}) = \max\{f(x^{(i)})\}\}}$}
    \For{$i \in [1..\lambda']$} \Comment{Phase 2: Crossover}
        \State{$y^{(i)} \gets \Call{Crossover}{x, x', c}$}
    \EndFor
    \State{$y \gets \Call{UniformRandom}{\{y^{(j)} \mid f(y^{(j)}) = \max\{f(y^{(i)})\}\}}$}
    \If{$f(y) > f(x)$} \Comment{Selection and Adaptation}
        \State{$x \gets y$, $\lambda \gets \max\{\lambda / F, 1\}$}
    \ElsIf{$f(y) = f(x)$}
        \State{$x \gets y$, $\lambda \gets \min\{\lambda F^{1/(U - 1)}, \lambdabound\}$}
    \Else
        \State{$\lambda \gets \min\{\lambda F^{1/(U - 1)}, \lambdabound\}$}
    \EndIf
\EndFor
\end{algorithmic}
\end{algorithm}

The main idea of the self-adjusting \ga is as follows. If an iteration leads to an increase of the fitness, indicating that progress is easy, then the value of $\lambda$ is reduced by a constant factor $F > 1$. If an iteration did not produce a fitness improvement, then $\lambda$ is increased by a factor of $F^{1/4}$. Consequently, after a series of iterations with an average success rate of $1/5$, the algorithm ends up with the initial value of $\lambda$. Needless to say, $\lambda$ can never drop below $1$ and never rise above $n$ (when using the recommended mutation rate $\lambda/n$). Since we will later regard a self-adaptive version with different upper limit, we formulated Algorithm~\ref{algo:ll-adjusting} already with the additional parameter $\lambdabound$ as upper limit. Hence the self-adaptive \ga as proposed in~\cite{learning-from-black-box-thcs} uses $\lambdabound=n$. We also note that whenever $\lambda$ should be interpreted as an integer (e.g. when the population size of the current iteration needs to be determined), the value $\lambda' = [\lambda]$ rounded to the closest integer is taken instead.

\subsection{Chernoff Bounds and McDiarmid's Inequalities}

In this section we describe the tools from the probability theory which we use in this paper.

Suppose $X_1, \ldots, X_n$ are independent random variables taking values in $\{0, 1\}$. Let $X$ denote their sum, and let $\mu = E[X]$ denote the sum's expected value. Then, for any $\delta > 0$, the following bounds are known as (simplified multiplicative) \emph{Chernoff bounds}.
\begin{align*}
&\Pr[X \ge (1 + \delta) \mu] \le e^{-\frac{\delta^2 \mu}{3}}, &&0 < \delta \le 1; \nonumber\\
&\Pr[X \ge (1 + \delta) \mu] \le e^{-\frac{\delta \mu}{3}}, &&1 \le \delta; \nonumber\\
&\Pr[X \le (1 - \delta) \mu] \le e^{-\frac{\delta^2 \mu}{2}}, &&0 < \delta < 1. \label{chernoff-le}
\end{align*}

Suppose now that the $X_i$ may take values in $[-1,1]$, however, such that for some $\kappa > 0$ we have $\Pr[X_i = 0] \ge 1-\kappa$ for all $i \in [1..n]$. For this situation, we prove the following corollary from the Bernstein's inequality, which seems to be rarely used in the theory of evolutionary computation. 
\begin{equation}
  \Pr[X \le \mu - \delta] \le \max\{\exp(-\delta^2 / 4 \kappa n), \exp(-\tfrac 38 \delta)\} \label{eq:chernoffvar}
\end{equation}

\begin{proof}
  By the assumption that $\Pr[X_i = 0] \ge 1-\kappa$, we have $\Var[X_i] \le \kappa$ and thus $\Var[X] \le \kappa n$. Hence Bernstein's~\cite{Bernstein1924} inequality (see also~\cite[Theorem 1.12ff]{augerdoerr}) yields 
\begin{align*}
\Pr[X \le E[X] - \delta] &\le \exp(-\delta^2 / (2 \Var[X] + \tfrac 43 \delta)) \\ 
&\le \exp(-\delta^2 / 2 \max\{2 \Var[X], \tfrac 43 \delta\}) \\
&= \max\{\exp(-\delta^2 / 4 \kappa n), \exp(-\tfrac 38 \delta)\}.\qedhere
\end{align*}
\end{proof}

Suppose $X_1, \ldots, X_n$ are arbitrary independent random variables, and assume the function $f$ defined over the product of the domains of the $X_i$ satisfies 
\begin{equation*}
\sup_{x_1, x_2, \ldots, x_n, \hat{x}_i} |f(x_1, \ldots, x_n) - f(x_1, \ldots, \hat{x}_i, \ldots, x_n)| \le c_i
\end{equation*}
for all $1 \le i \le n$. Then, for all $\delta > 0$, the following \emph{McDiarmid's inequalities}~\cite{mcdiarmid-inequality}
hold.
\begin{align}
&\Pr[f(X_1, \ldots, X_n) - E[f(X_1, \ldots, X_n)] \ge +\delta] \le e^{-\frac{2\delta^2}{\sum_{i=1}^{n}{c_i^2}}}; \label{mcdiarmid-ge}\\
&\Pr[f(X_1, \ldots, X_n) - E[f(X_1, \ldots, X_n)] \le -\delta] \le e^{-\frac{2\delta^2}{\sum_{i=1}^{n}{c_i^2}}}. \nonumber
\end{align}

%% file: parts/3-onell-on-sat.tex
\section{Almost Like OneMax:\\Conditions for the $\onell$~GA\\to Behave Well on a Random 3-CNF Formula}

In this section, we formulate and prove conditions which are enough for the $\onell$~GA to have local fitness-distance correlation.

In the Doerr, Neumann and Sutton paper~\cite{doerr-neumann-sutton-oneplusone-cnf}, two results were proven. When the clause density $m / n$ is $\Omega(n)$, then a strong fitness distance correlation is exhibited. Apart from an exponentially small failure probability, the random instance is such that any possible one-bit flip leads to a fitness gain or loss (depending on whether the distance decreases or not) of $\Theta(m/n)$. In particular, the fitness function is such that there are no misleading one-bit flips. On such a function, the classic ``multiplicative drift with the fitness'' proof for \onemax can be imitated and easily yields an $O(n \log n)$ optimization time (in expectation and with high probability)~\cite[Theorem 2]{doerr-neumann-sutton-oneplusone-cnf}.

For smaller clause densities $m / n = \Omega(\log n)$, with probability $1 - o(1)$, taken over both the random 3-CNF formula and the random decisions of the optimization process, also an $O(n \log n)$ optimization time is observed~\cite[Theorem 3]{doerr-neumann-sutton-oneplusone-cnf}. The difference to the previous setting is that now there may be misleading one-bit flips, however, they are rare enough that a typical run does not encounter them.

In both these theorems, the \oea with mutation rate $1/n$ is regarded. Consequently, with constant probability exactly one bit is flipped. This together with the fitness-distance correlations exhibited is exploited to compute the drift. In contrast to this, the $\onell$~GA uses larger mutation rates, meaning that the typical distance between parent and offspring is larger. In addition, it does not need to estimate whether parent or offspring is closer to the optimum, but it needs to select among several offspring the one with slightly smaller distance. 

For example, in the case $\lambda=10$ consider a certain point in time when the parent individual has already a decent fitness. Then most offspring in the mutation phase will have a genotypic distance that is worse than that of the parent by a margin of 10. An offspring which has flipped a single bit which is missing in the parent will have a distance of 8. To be successful, this 20\% advantage has to be visible for the $\onell$~GA via a sufficiently strong fitness-distance correlation.

The common sense suggests that the bigger the $\lambda$ is, the more problems the algorithm should have in detecting a ``good'' offspring. On the other hand, the bigger the clause density $m / n$ is, the simpler it should be for the algorithm to handle bigger values of $\lambda$, as the problem becomes more similar to \textsc{OneMax}.

First, we investigate how the average fitness of a search point at a distance $d$ from the optimum looks like, and how a difference of two such values behaves asymptotically depending on the difference between distances.

\begin{lemma}\label{lemma:exact-average-probability}
    The probability for a random 3-CNF clause $C$, consisting of distinct variables and satisfied by the planted assignment $x^{*}$, to be also satisfied by an assignment $x$ with the Hamming distance $d = d(x)$ from $x^{*}$, is
\begin{equation*}
    P(n, d) = \frac{6 \cdot \binom{n}{3} + \binom{n - d}{3}}{7 \cdot \binom{n}{3}}.
\end{equation*}
\end{lemma}
\begin{proof}
    There are $\binom{n}{3}$ ways to choose a set of three different variables. For three fixed variables, there are 7 ways to choose their signs in a way that the resulting clause is satisfied by $x^*$. Hence the set $\CC^*$ of all clauses satisfied by $x^*$ has cardinality $|\CC^*| = 7\cdot\binom{n}{3}$. For a clause $C \in \CC^*$ to be not satisfied by $x$, we need that $C$ contains at least one variable in which $x$ and $x^*$ differ; there are exactly $\binom{n}{3} - \binom{n-d}{3}$ such sets of three variables. For each such set, there is exactly one way of setting their signs in such a way that the resulting clause is not satisfied by $x$ (and these signs make the clause satisfied by $x^*$). Consequently, there are exactly $\binom{n}{3} - \binom{n-d}{3}$ clauses in $\CC^*$ which are not satisfied by $x$, giving the claim.
\end{proof}

\begin{lemma}\label{lemma:asymptotic-difference}
    If $n - d = \Theta(n)$ and $\ell = o(n)$, then $P(n, d) - P(n, d + \ell) = \Theta(\ell / n).$
\end{lemma}
\begin{proof}
    From Lemma~\ref{lemma:exact-average-probability}, we compute
    \begin{align*}
        P(n, d) - P(n, d + \ell) &= \frac{\binom{n-d}{3} - \binom{n-d-\ell}{3}}{\binom{n}{3}} \\
        &= \frac{3(n-d)^2\ell -3(n-d-1)\ell^2 - 6(n-d)\ell + \ell^3 + 2\ell}{n(n-1)(n-2)} \\
        &= \frac{\Theta(n^2) \cdot \ell + \Theta(n) \cdot \ell^2 + \ell^3 \pm o(n^2)}{\Theta(n^3)} = \Theta(\ell / n),
    \end{align*}
    where we used that $n \ell^2 = o(n^2 \ell)$ and $\ell^3 = o(n^2 \ell)$, which both follow from the assumption $\ell=o(n)$.
\end{proof}

\begin{corollary}\label{corollary:fitness-averages}
    Consider a random 3-CNF formula on $n$ variables and $m$ clauses with the planted assignment $x^{*}$. The expected fitness $f_{\avg}(d)$ of any search point $x$ with the Hamming distance $d = d(x)$ from $x^{*}$ is
    \begin{equation*}
        f_{\avg}(d) = m \cdot \frac{6 \cdot \binom{n}{3} + \binom{n - d}{3}}{7 \cdot \binom{n}{3}}.
    \end{equation*}
    For $n - d = \Theta(n)$ and $\ell = o(n)$, we have
    \begin{equation*}
        f_{\avg}(d) - f_{\avg}(d - \ell) = \Theta(m \ell / n).
    \end{equation*}
\end{corollary}
\begin{proof}
    Follows trivially from the definition of fitness, Lemmas~\ref{lemma:exact-average-probability} and~\ref{lemma:asymptotic-difference}, and linearity of expectation.
\end{proof}

As we see, the average fitness values demonstrate a good fitness-distance correlation. However, the actual fitness values of concrete search points may deviate quite far apart from the average values when the concrete 3-CNF formula is fixed. To show a good performance of the $\onell$ GA, we now show a stronger fitness-distance correlation for offspring from the same parent. For a number $\lambdabound$ to be made precise later, we define the following.

%However, the $\onell$~GA creates offspring, which it subsequently compares, by changing the same number of bits in a single parent, and this number of bits is typically not very large. We show that this introduces large deviations rarely enough, so that the selection scheme of the $\onell$~GA still works with at least constant probability. The following definition formalises what we think is ``rarely enough''.

\begin{definition}\label{def:well-behaved}
    Consider a search point $x$ with a distance $d > 0$ from the planted assignment $x^{*}$. Let $\ell \in [1..\lambdabound]$.
    \begin{itemize}
        \item The set $X_{\ell}^{-} := X_{\ell}^{-}(x)$ of \emph{$\ell$-bad offsping} is the set of all search points produced from flipping in $x$ exactly $\ell$ bits which coincide in $x$ and $x^{*}$.
        \item The set $X_{\ell}^{+} := X_{\ell}^{+}(x)$ of \emph{$\ell$-good offspring} is the set of all search points produced from flipping in $x$ exactly $\ell$ bits of which at least one is different in $x$ and $x^{*}$.
        \item The point $x$ is \emph{well-behaved} if for all $\ell \in [1..\lambdabound]$ 
        \begin{enumerate}
        \item[(i)] there are at most $|X_{\ell}^{-}| / \lambdabound$ elements $x^{-} \in X_{\ell}^{-}$ such that $f(x) - f(x^{-}) \le f_{\avg}(d) - f_{\avg}(d + \ell - 1)$, and
        \item[(ii)] there are at most $|X_{\ell}^{+}| / 2$ elements $x^{+} \in X_{\ell}^{+}$ such that $f(x) - f(x^{+}) \ge f_{\avg}(d) - f_{\avg}(d + \ell - 1)$.
        \end{enumerate}
    \end{itemize}
\end{definition}

The motivation for this definition is that whenever the current search point $x$ is well-behaved, the selection inside the mutation phase of the $\onell$~GA is able to distinguish with good probability an offspring with one of the missing bits guessed right from the offspring with no missing bit flipped.

\begin{lemma}\label{lemma:selection-works}
    Assume $x$ is the current best solution of the $\onell$~GA, $f(x) < m$, $x$ is well-behaved and the current value of $\lambda$ satisfies $\lambda \le \lambdabound$ for an integer $\lambdabound$. Then, whenever at least one good offspring appears at the mutation phase, the $\onell$~GA will choose one of them with at least constant probability as mutation winner.
\end{lemma}
\begin{proof}
    Assume that, in the current step, $\ell$ was sampled from $\mathcal{B}(n, \lambda / n)$ in such a way
that $\ell \le \lambdabound$. As the median of $\mathcal{B}(n, \lambda / n)$ is at most $\lceil \lambda \rceil$, and $\lceil \lambda \rceil \le \lambdabound$ as $\lambdabound$ is integer, this happens with probability of at least $1/2$.

    Assume that among the \revisepar{$[\lambda]$}{Can we avoid the notation $[x]$? Many people use $[x]=[1..x]$. Can we just round up the $\lambda$?} offspring created in the mutation phase, $g \ge 1$ are $\ell$-good and $[\lambda] - g$ are $\ell$-bad. The probability that all $\ell$-bad offspring have a fitness value below $f(x) - (f_{\avg}(d(x)) - f_{\avg}(d(x) + \ell - 1))$ is at least
    \begin{equation*}
        \left(1 - \frac{1}{\lambdabound}\right)^{[\lambda] - g} \ge \left(1 - \frac{1}{\lambdabound}\right)^{{\lambdabound} - 1} \ge e^{-1}.
    \end{equation*}
    The probability of, say, the first $\ell$-good offspring to have a fitness value above $f(x) - (f_{\avg}(d(x)) - f_{\avg}(d(x) + \ell - 1))$ is at least $1/2$. Thus, the probability that the $\onell$~GA chooses an offspring which is not $\ell$-bad, whenever it exists, is at least $e^{-1} / 4$.
\end{proof}

Now we show that for all sufficiently small values of $\lambdabound$, all interesting search points $x$ are well-behaved with an overwhelming probability. In a sense, this is the core of the main result of this paper.

\begin{theorem}\label{theorem:everyone-behaves-well}
    If $\lambdabound$ is an integer satisfying $\lambdabound = o(\min\{n, (m / n)^{1/4}\})$, then with probability $e^{-\omega(n)}$ (taken over possible 3-CNF formulas satisfying the planted assignment $x^{*}$) all $x$ with $n - d(x) = \Theta(n)$ are well-behaved.
\end{theorem}
\begin{proof}
    There are $2^n$ possible search points $x$ in total (including the search points which are too far from the optimum), and for every search point $2\lambdabound = O(n)$ statements about sets $X_{\ell}^{-}(x)$ and $X_{\ell}^{+}(x)$ need to be proven. So by a union bound, it is enough to prove that for each $\ell \in [1..\lambdabound]$ and each $x$ with probability $e^{-\omega(n)}$ the sets $X_{\ell}^{-}(x)$ and $X_{\ell}^{+}(x)$ are as in the definition of well-behaved.%, because by union bound the lemma statement fails with probability $2^n \cdot \Theta(n) \cdot e^{-\omega(n)} = e^{-\omega(n) + \Theta(n)} = e^{-\omega(n)}$.
    
    For the remainder of this proof, we fix an $x$ and a value of $\ell$. Let $d := d(x)$. We prove only the statement on $X^-(x)$. The statement on $X^+(x)$ can be proven in the same way and then even yields a sharper bound as required in Definition~\ref{def:well-behaved}.
    
    For convenience, we use the short-hand notation $X^{-} = X_{\ell}^{-}(x)$. For an $x^{-} \in X^{-}$, the fitness loss $f(x) - f(x^-)$ is a random variable which is determines by the random formula $F$. We denote by $f^{-} = f_{\avg}(d + \ell)$ the average value of $f(x^{-})$ and by $f^{=} = f_{\avg}(d + \ell - 1)$ the threshold value from Definition~\ref{def:well-behaved}. We use an indicator random variable $V(x^{-})$ for the event that $x^-$ violates the fitness constraint in Definition~\ref{def:well-behaved}, that is, 
    \begin{equation*}
    V(x^{-}) := \begin{cases}
        1, & f(x) - f(x^{-}) \le f_{\avg}(d) - f^{=}, \\
        0, & \text{otherwise}.
    \end{cases}
    \end{equation*}
    
    The statement we want to prove is equivalent to 
    \begin{equation}
    \Pr\left[\,\sum_{x^{-} \in X^{-}} V(x^{-}) > \frac{|X^{-}|}{\lambdabound}\right] < e^{-\omega(n)}. \label{x-is-okay-part1}
    \end{equation}
    
    This expression calls for one of the Chernoff bounds to be applied. However, the resulting bounds are not strong enough for our aim to be achieved. For this reason, we introduce a new random variable, $\tilde{V}(x^{-})$, which is less sensitive to the change of a single clause in the random instance.
    Let $f^{\equiv} = (f^{-} + f^{=}) / 2$. Then
    \begin{equation*}
     \tilde{V}(x^{-}) := \begin{cases}
         0, & f(x) - f(x^{-}) \ge f_{\avg}(d) - f^{\equiv}; \\
         \frac{f(x) - f(x^{-}) - (f_{\avg}(d) - f^{=})}{f^{=} - f^{\equiv}}, & f_{\avg}(d) - f^{\equiv} \ge f(x) - f(x^{-}) \ge f_{\avg}(d) - f^{=}; \\
         1, & f(x) - f(x^{-}) \le f_{\avg}(d) - f^{=}.
    \end{cases}
    \end{equation*}
     
     %In simple words, $\tilde{V}(x^{-})$ starts growing linearly when $f(x^{-})$ is halfway from $f^{-}$ to $f^{=}$. 
    As $\tilde{V}(x^{-}) \ge V(x^{-})$ with probability one, it is enough to show, instead of~\eqref{x-is-okay-part1}, that
    \begin{equation}
     \Pr\left[\,\sum_{x^{-} \in X^{-}} \tilde{V}(x^{-}) > \frac{|X^{-}|}{\lambdabound}\right] < e^{-\omega(n)}. \label{x-is-okay-part1-tilde}
    \end{equation}
     
    For brevity, we write $\tilde{V}^{\Sigma} := \sum_{x^{-} \in X^{-}} \tilde{V}(x^{-})$. We interpret the random variable $\tilde{V}^{\Sigma}$ as a function of the independent uniformly distributed random variables $C_1, \ldots, C_m$ which define the random formula $F$. Recall that, by Corollary~\ref{corollary:fitness-averages}, $f^{=} - f^{-} = \Theta(m/n)$, and so is $f^{=} - f^{\equiv}$. Thus replacing a single $C_i$ in $F$ by $C'_i$ introduces a change of at most $2 / (f^{=} - f^{\equiv}) = \Theta(n / m)$ in the value $\tilde V(x^-)$, however, only for those $x^{-}$ for which $x$ and $x^-$ differ in one of the at most $6$ variables contained in $C_i \cup C'_i$ (for all other $x^-$, the value of $\tilde V(x^-)$ does not change). The number of the former kind of $x^-$ is at most $\binom{n - d}{\ell} - \binom{n - d - 6}{\ell} = \sum_{i=0}^5 \big(\binom{n-d-i}{\ell} - \binom{n-d-(i+1)}{\ell}\big) = \sum_{i=0}^5 \binom{n-d-(i+1)}{\ell-1} \le 6 \binom{n-d-1}{\ell-1} = 6 \binom{n - d}{\ell} \frac{\ell}{n - d} = |X^{-}| \cdot\Theta(\ell / n)$. The maximum change of $\tilde{V}^{\Sigma}$ inflicted by changing one clause hence is $|X^{-}| \cdot O(\ell / m)$.
      
    We use this statement to apply McDiarmid's inequality, see~\eqref{mcdiarmid-ge}, and obtain
    \begin{align}
      \Pr\left[\tilde{V}^{\Sigma} > |X^{-}| / \lambdabound\right]
        &= \Pr\left[\tilde{V}^{\Sigma} > E\left[\tilde{V}^{\Sigma}\right] + \left(|X^{-}| / \lambdabound - E\left[\tilde{V}^{\Sigma}\right]\right)\right] \nonumber\\
        &\le \exp\left(-\frac{2\left(|X^{-}| / \lambdabound - E\left[\tilde{V}^{\Sigma}\right]\right)^2}{m \cdot (|X^{-}| \cdot O(\ell / m))^2}\right) \nonumber\\
        &=   \exp\left(-\Omega\left(\frac{m \cdot \left(1 - \lambdabound \frac{E[\tilde{V}^{\Sigma}]}{|X^{-}|}\right)^2}{\ell^2 \lambdaboundsq}\right)\right). \label{failure-prob-raw}
    \end{align}
      
    To complete this bound, we need to show that the term $(1 - \lambdabound E[\tilde{V}^{\Sigma}] / |X^{-}|)$ is at least some positive constant. To do this with the least effort, we introduce the random variable $\overline{V}(x^{-})$ defined by
    \begin{equation*}
    \overline{V}(x^{-}) = \begin{cases}
          1, & f(x) - f(x^{-}) \le f_{\avg}(d) - f^{\equiv}, \\
          0, & \text{otherwise}
    \end{cases}
    \end{equation*}
    for all $x^{-} \in X^-$, and observe that it dominates $\tilde{V}(x^{-})$. Consequently, we have $E[\tilde{V}(x^{-})] \le E[\overline{V}(x^{-})] = \Pr[f(x) - f(x^{-}) \le f_{\avg}(d) - f^{\equiv}]$. By symmetry, these probabilities are identical for all $x^{-} \in X^-$. As $E[f(x^{-})] = f^{-}$, $f^{\equiv} - f^{-} = \Theta(m / n)$, and \startnewtext $E[f(x) - f^{-}] = \Theta(\ell m / n)$\finishnewtext, the bound from~\eqref{eq:chernoffvar} gives
    \begin{align*}
      \Pr[f(x) &- f(x^{-}) \le f_{\avg}(d) - f^{\equiv}] \\
          &= \Pr[f(x) - f(x^{-}) \le E[f(x)] - E[f(x^{-})] + f^{-} - f^{\equiv}] \\
          &= \Pr[f(x) - f(x^{-}) \le E[f(x) - f(x^{-})] - \Theta(m/n)] \\
          &\le \max\{\exp(-\Omega(m^2/n^2) / (12\ell m/n)), \exp(-(3/8) \Omega(m/n))\}\\
          &= \max\{\exp(-\Omega(m/\ell n)), \exp(-\Omega(m/n))\}\\
          &= \exp(-\Omega(m/\ell n)).
    \end{align*}
      
    Note that $f(x) - f(x^{-})$ is the sum of $m$ independent random variables describing the influence of each of the $m$ random clauses on this expression. With probability at least $1-3\ell/n$, a random clause contains none of the $\lambda$ variables $x$ and $x^-$ differ in. In this case, the clause contributes equally to $f(x)$ and $f(x^-)$, hence zero to the difference.
    
    \startnewtext
    To finish this part, we shall note that we can estimate the subtrahend in the term $(1 - \lambdabound E[\tilde{V}^{\Sigma}] / |X^{-}|)$ as follows:
    \begin{align*}
        \lambdabound \cdot \frac{E[\tilde{V}^{\Sigma}]}{|X^{-}|}
        &= \lambdabound \cdot \frac{\sum_{x^{-} \in X^{-}} E[\tilde{V}(x^{-})]}{|X^{-}|}
        \le \lambdabound \cdot \frac{\sum_{x^{-} \in X^{-}} E[\overline{V}(x^{-})]}{|X^{-}|}
        \\&= \lambdabound \cdot E[\overline{V}(x^{-})] = \lambdabound \cdot \exp(-\Omega(m / \ell n))
        \le \lambdabound \cdot \exp(-\Omega(m / \lambdabound n)).
    \end{align*}

    By the theorem statement, we safely assume that $\lambdabound = o((m / n)^{1/4})$. This allows further refining the estimation:
    \begin{equation*}
        \lambdabound \cdot \exp(-\Omega(m/\lambdabound n)) =
        o\left(\left(\frac{m}{n}\right)^{1/4}\right) \cdot \exp\left(-\omega\left(\left(\frac{m}{n}\right)^{3/4}\right)\right) = o(1),
    \end{equation*}
    which means that $1 - \lambdabound E[\tilde{V}^{\Sigma}] / |X^{-}| = 1 - o(1)$.
    \finishnewtext
    Consequently, from~\eqref{failure-prob-raw} we obtain
    \begin{equation*}
    \Pr\left[\tilde{V}^{\Sigma} > |X^{-}| / \lambdabound\right] \le \exp\left(-\Omega\left(\frac{m}{\ell^2 \lambdaboundsq}\right)\right) \le e^{-\omega(n)},
    \end{equation*}
    where the second estimate follows from $\ell \le \lambdabound$ and $\lambdabound = o((m / n)^4)$. This finished the proof.
\end{proof}

Now we are ready to prove the main result of this section.

\begin{theorem}\label{3cnf-bounding-theorem}
    Consider solving random 3-CNF formulas with planted solutions on $n$ variables and $m$ clauses by the $\onell$~GA, where $m / n > c \log n$ for large enough constant $c$. If there exists some integer $\lambdabound = o(\min\{n, (m / n)^{1/4}\})$ such that, during the entire run of the $\onell$~GA, $\lambda \le \lambdabound$, and the algorithm starts from a random assignment, then with probability $1 - o(1)$, taken both over the random 3-CNF formulas and the algorithm decisions, the algorithm will demonstrate, in every point, the same progress, divided by at most constant, as the same algorithm optimising \textsc{OneMax} on $n$ variables.
\end{theorem}
\begin{proof}
    With the exponentially small probability of failure, the algorithm starts in a search point which is at most $(1/2 + \varepsilon) \cdot n$ bits apart from the planted assignment for a constant $\varepsilon > 0$. \startnewtext Thus, the $n - d(x) = \Theta(n)$ condition is satisfied for the initial search point, which enables Lemma~\ref{lemma:asymptotic-difference} and its numerours corollaries. This means that the initial point is well-behaved by Theorem~\ref{theorem:everyone-behaves-well}.
    
    With a well-behaved parent, due to Lemma~\ref{lemma:selection-works}, the mutation phase selects an offspring which contains one of the missing bits with the probability, which is at most a constant factor worse than the probability of the same happening with \textsc{OneMax}. With probability of at least $1 - (1 - 1 / (e\ell))^{\ell} \ge 1 - (1 - 1 / (e\lambdabound))^{\lambdabound} \ge 1 - e^{-1/e}$, one of the missing bits from the parent, which is found in the chosen offspring, gets recombined with the parent and thus produces an assignment which is one bit closer to the optimum. Finally, by Lemma~\ref{lemma-one-bit-flip-is-ok}, this assignment has a better fitness than the parent, thus it is accepted with probability at least $1 - n^{-3}$. On the other hand, \revisepar{if no good bits were found}{this is slippy}, the same lemma ensures that an assignment which is more distant from the optimum is not accepted with the same probability. This means that the new parent still satisfies the $n - d(x) = \Theta(n)$ condition, thus it is well-behaved. Repeated application of this paragraph, until the optimum is found, proves the theorem. \finishnewtext
\end{proof}

As a consequence, for logarithmic clause density ($m / n = \Theta(\log n)$) the upper bound on $\lambda$, in order for the $\onell$~GA to work consistently, is $o((\log n)^{1/4})$.

%% file: parts/4-bounded-onell.tex
\section{Running Time of the Adaptive $\onell$~GA with Constrained $\lambda$ on OneMax}

In this section, we prove the bounds on the running time on \textsc{OneMax} of the adaptive version of the $\onell$~GA, provided that the adaptation of the population size $\lambda$ uses the 1/5-th rule, but keeps $\lambda \le \lambdabound$ for some $\lambdabound$, which is possibly a function of the problem size~$n$.

In~\cite{learning-from-black-box-thcs}, it was shown that the optimal value of $\lambda$, given the current fitness value $f(x)$, is $\lambda^{*} = \sqrt{n / (n - f(x))}$. The 1/5-th rule, as seen in experiments~\cite{learning-from-black-box-thcs}, tends to keep $\lambda$ close to this optimum value, and this effect has been proven in~\cite{doerr-doerr-lambda-lambda-self-adjustment}.

When $\lambda$ is bounded, the overall algorithm run can be split into roughly two phases. The first phase is when the adaptation largely determines the value of $\lambda$, which follows from the deserved $\lambda^{*}$ being (much) less than $\lambdabound$. The second phase starts when $\lambda$ finally hits the $\lambdabound$, and while $\lambda^{*}$ grows, the actual $\lambda$ has to stay close to $\lambdabound$.

As a consequence, the runtime of the first phase is, roughly, the same as the runtime of the adaptive $\onell$~GA on the same range of fitness values, and the runtime of the second phase is, again roughly, the same as the runtime of the fixed-size $\onell$~GA on the same range of fitness values. We formalize these ideas in the proof of the following theorem.

\begin{theorem}
The expected running time of the adaptive $\onell$~GA on \textsc{OneMax}, where the population size parameter $\lambda$ is bounded from above by some $\lambdabound = \lambdabound(n)$, is $O(n \cdot \max(1, (\log n) / \lambdabound))$.
\end{theorem}

\begin{proof}
We consider two \emph{stages} of the optimization process. The first stage starts when the optimization starts and ends when the distance to optimum $d(x)$, which is $n - f(x)$, first becomes at most $\lfloor n / (2 \lambdaboundsq) \rfloor$. The second stage starts at this point and ends when the optimization ends. We show that the runtime of the first stage is $O(n)$, and the runtime of the second stage is $O(n (\log n) / \lambdabound)$.

\paragraph{First stage}
Here, we basically repeat the proof of~\cite[Theorem 5]{doerr-doerr-lambda-lambda-self-adjustment-arxiv}, with the following corrections:
\begin{itemize}
    \item The original proof uses a threshold for the $\lambda$, which is $C_0 \cdot \lambda^{*}$ for the optimal $\lambda^{*} = \sqrt{n / (n - f(x))}$ and a certain large enough constant $C_0$. Above this threshold, the probability for the algorithm to find the next meaningful bit, and thus to improve the best seen fitness, is so high, that the random process of changing $\lambda$ drifts strong enough towards decreasing $\lambda$. Here, we need to change this threshold to $\min(\lambdabound, C_0 \cdot \lambda^{*})$, where $\lambdabound$ is the constraint on the population size. Note that the notation $\lambdabound$ means a different thing in~\cite{doerr-doerr-lambda-lambda-self-adjustment-arxiv} (in that paper, it is a synonym of $C_0 \cdot \lambda^{*}$).
    \item Claim 2.1~still reads and is proven in the same way, however, whenever $C_0 \cdot \lambda^{*} \cdot F^{t/4}$ exceeds our constraint $\lambdabound$, instead of the exact population size, we start talking about its upper bound.
    \item Claim 2.2~refers to Lemma~6, which may stop working if population size hits the $\lambdabound$. However, just at the same moment of time, it is no longer possible to increase the population size, so the drift of $\lambda$ towards greater values stops, but its negative counterpart remains available to the algorithm. This means that the expected progress of the random walk of $\log_{F} \lambda$ still becomes negative whenever $\lambda$ exceeds either $\lambdabound$ (in our notation) or $C_0 \cdot \lambda^{*}$, and is still separated from zero by a constant.
\end{itemize}

This means that the upper bound on the running time of the adaptive $\onell$~GA on the interval $f(x) \in [\lfloor n / (2 \lambdaboundsq) \rfloor..n]$, which is $O(n)$, can be translated without any major change to the adaptive version with the constrained population size.

\paragraph{Second stage}
Similarly to the proof of~\cite[Theorem 5]{doerr-doerr-lambda-lambda-self-adjustment-arxiv}, we split the process in \emph{phases}. Each phase starts just after a new fitness level is reached and ends with selecting an individual with the fitness value strictly greater than the one of the parent. As the adaptation will necessarily drop $\lambda$ by the adaptation strength factor of $F$ right before the phase starts, the initial value for the $\lambda$ is $\lambda_0 \le \lambdabound / F$.

We differentiate between \emph{short} and \emph{long} phases. A \emph{short} phase always satisfies $\lambda < \lambdabound$. In exactly the same way as in~\cite[Theorem 5, Claim 1]{doerr-doerr-lambda-lambda-self-adjustment-arxiv}, we are able to prove that the expected running time of a \emph{short} phase is $O(\lambdabound)$.

A \emph{long} phase necessary has at least one iteration, at the end of which $\lambda = \lambdabound$. What is more, the equality $\lambda = \lambdabound$ holds until the very end of this phase, as the only possibility for $\lambda$ to drop down is generation of an individual strictly better than its parent, which signifies the end of the phase. For convenience, we split every \emph{long} phase in an \emph{initial} phase, which proceeds until $\lambda = \lambdabound$, and a \emph{main} phase, which maintains $\lambda = \lambdabound$.

The running time of every \emph{initial} phase is $O(\lambdabound)$, as for \emph{short} phases. This means that the common runtime of all \emph{short} and \emph{initial} phases altogether, since $d(x) \le \lfloor n / (2 \lambdaboundsq) \rfloor$, can be estimated as $n / (2 \lambdaboundsq) \cdot O(\lambdabound) = O(n / (2 \lambdabound)) = O(n)$, because, for every fixed algorithm run, every \emph{long} or \emph{short} phase corresponds to exactly one unique $f(x)$.

All \emph{main} phases feature $\lambda = \lambdabound$, which means their common runtime is at most the same as the runtime of the $\onell$~GA with fixed $\lambda = \lambdabound$. This means it can be estimated in exactly the same way as in~\cite[Theorem 2, upper bound, third phase]{doerr-doerr-lambda-lambda-fixed-tight}, which gives a bound of $O(n (\log n) / \lambdabound)$. This bound determines the runtime of the second stage.
\end{proof}

This theorem, although being interesting on its own, can be also applied to solving random 3-SAT problems instead of \textsc{OneMax}, provided $\lambdabound$ is small enough to satisfy the conditions of Theorem~\ref{3cnf-bounding-theorem}, and keeping in mind that this bound on the \emph{expected} runtime holds, in the case of 3-SAT, with probability $1 - o(1)$.

It follows particularly that an upper bound of $\Theta(\log n)$ on the population size still brings the linear runtime of the adaptive $\onell$~GA on \textsc{OneMax}. Recall that, when the adaptation is not limited, the population size grows up to $\Theta(\sqrt{n})$ towards the last iterations of the algorithm. The new result shows that such large population sizes are not really necessary for the success on \textsc{OneMax}, although the corresponding insight could be seen already in~\cite{doerr-doerr-lambda-lambda-fixed-tight}.

%% file: parts/5-experiments.tex
\section{Experiments}

Although theoretical results give us certain insights about how powerful our adaptation schemes are, they, at least in their current form, may be not enough to persuade practitioners whether they are of any use in solving practical problems. To complement our theoretical research, we conducted a series of experiments, which appear to demonstrate that application of our ideas and proposals yields immediate and noticeable response on feasible problem sizes.

\subsection{Results for OneMax}

The first experiment was dedicated to evaluation of performance for fixed-size $\onell$~GA, as well as unlimited and constrained adaptation schemes, on the \textsc{OneMax} problem. Similar experiments had been also performed in~\cite{learning-from-black-box-thcs}. In this paper, we evaluated much larger problem sizes, which helped revealing subtle differences in behaviour of various algorithms. We also evaluated more fixed population sizes systematically, and we also included the version of the $\onell$~GA with constrained adaptation ($\lambda \le 2 \log (n + 1)$).

We considered problem sizes to be powers of two, from $2^4$ to $2^{24}$. The fixed population sizes were chosen to be $\lambda \in [2..20]$, to cover various ranges of behaviour. For the adaptive $\onell$~GA schemes, the unlimited one (denoted as $\lambda \le n$) and the logarithmically constrained one (denoted as $\lambda \le 2 \ln n$, in fact, the upper limit is $2 \log (n + 1)$) were included in the comparison. We also ran the $(1+1)$ EA. For every combination of the algorithm and the problem size, 100 runs were executed, and the median number of function evaluations is reported. As it is shown in~\cite{doerr-doerr-lambda-lambda-fixed-tight}, the runtimes of the fixed-size $\onell$~GA are well concentrated, and the same thing holds experimentally for the adaptive versions, the median plots are representative for demonstrating trends and differences for this problem.

The results are presented in Fig.~\ref{pic-onemax-plots}. The abscissa axis, logarithmically scaled, represents the problem size, while the ordinate axis, linearly scaled, represents the median number of function evaluations, divided by the problem size. In these axes, the $\Theta(n \log n)$ algorithms produce a straight line with the angle of $\alpha > 0$ to the abscissa axis, and the $\Theta(n)$ algorithms produce a horizontal line. Algorithms of intermediate complexities, such as $\Theta(n \sqrt{\log n})$, produce the plots, whose shapes depend on exponents at the logarithmic parts.

\begin{figure}[!t]
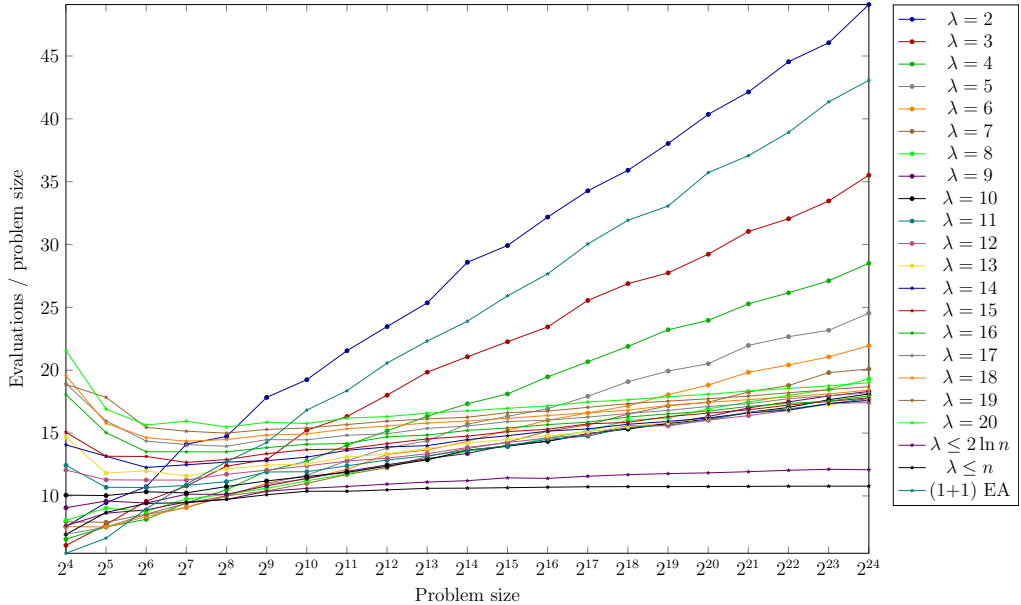

\scalebox{0.6}{\iqrPlotOneMax{1.4\textwidth}{\textwidth}{}}
\caption{Plots of median runtimes on OneMax}\label{pic-onemax-plots}
\end{figure}

Several phenomena, which has been previously shown only theoretically, can be seen in Fig.~\ref{pic-onemax-plots} as experimental plots for the first time. For instance, it can be seen that the fixed-size versions of the $\onell$~GA demonstrate a slightly superlinear behaviour at small sizes, and then switches at certain threshold to a strictly $\Theta(n \log n)$ behaviour. This behaviour follows from the $\Theta(\min(n \log n / \lambda, n \lambda \log \log \lambda / \log \lambda))$ bound proven in~\cite{doerr-doerr-lambda-lambda-fixed-tight}. In fact, we can even observe something similar to the $1 / \lambda$ quotient: all fixed-size plots have different inclines, seem to originate at the same point (which is slightly above the lower left corner of the plot), and the ratio at $n = 2^{24}$ is approximately 27 for $\lambda = 4$ and approximately $50$ for $\lambda = 2$, which is quite close to the 2x difference.

The lower envelope of the fixed-size plots should exactly correspond to the optimal fixed population size as a function of the problem size. We can clearly see that it behaves like a convex upwards function. Although it is impossible, at this scale, to determine the precise shape of the function (i.e. to spot the difference between the $\Theta(\sqrt{x})$ and $\Theta(\sqrt{x \log \log x / \log x})$, which corresponds to the difference between the original estimation in~\cite{learning-from-black-box-thcs} and the refined estimation in~\cite{doerr-doerr-lambda-lambda-fixed-tight}), it is clearly seen that this envelope is significantly sublinear, that is, the difference between $\Theta(n \log n)$ and the actual runtime of the optimal tuning (which is $\Theta(n \sqrt{\log n \log \log \log n / \log \log n})$) is seen from experiments.

Finally, both adaptive versions demonstrate easy-to-see linear runtime, with a quotient of approximately 11 for the unlimited adaptation, and of approximately 12.5 for the logarithmically constrained adaptation. Thus, our result shows that bounding the population size by a logarithmic bound not only preserves the asymptotics of the runtime, but also changes the absolute performance only very little. The difference between the adaptive and the optimally fixed-size versions can be, again, clearly seen from the experimental data.

\subsection{Results for Random 3-CNF Formulas}

The second experiment was dedicated to solving random 3-CNF formulas with planted solutions. The same set of algorithms was considered. We took the problem sizes again in the form of $2^t$, where $t \in [7..20]$. The rise of the lower bound is due to the fact that, at sizes $2^6$ and below, difficult formulas became too often to appear, and studying the behaviour of the $\onell$~GA on such formulas is not a scope of this research. The upper bound is chosen for performance reasons. We have evaluated the performance of the unlimited adaptive $\onell$~GA only for sizes up to $2^{16}$, because this algorithm was computationally the most expensive for evaluation, for the reasons which are discussed below. The number of clauses $m$ was a function of the problem size: $m(n) = \lfloor 4 n \log n \rfloor$.

The results are presented in Fig.~\ref{pic-3cnf-plots}. The trends seen here are quite similar to the ones for \textsc{OneMax} (see Fig.~\ref{pic-onemax-plots}), with two notable exceptions.

\begin{figure}[!t]
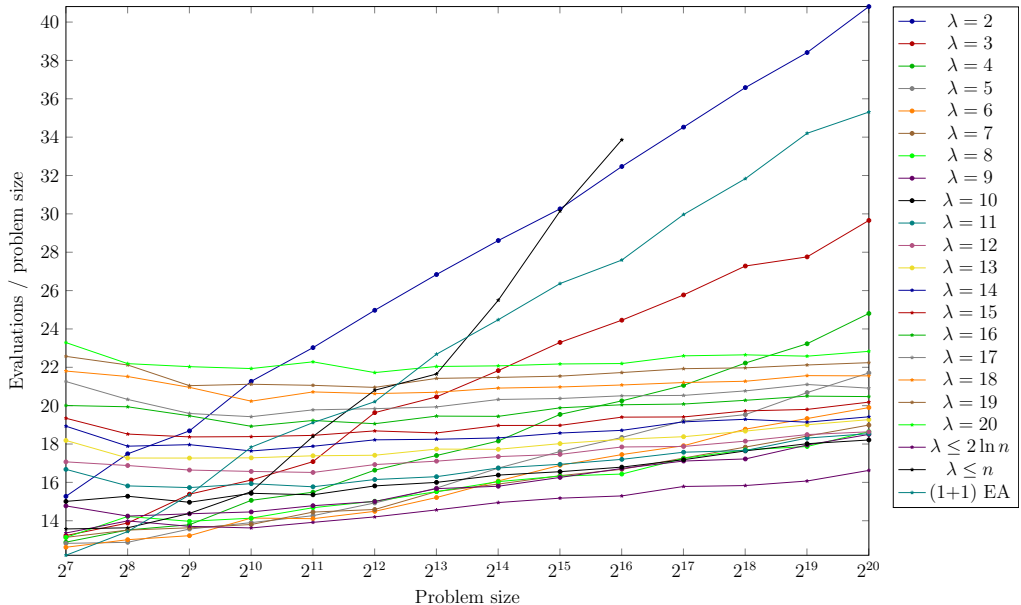

\scalebox{0.6}{\iqrPlotRandomCNF{1.4\textwidth}{\textwidth}{}}
\caption{Plots of median runtimes on random 3-CNF formulas with planted solutions}\label{pic-3cnf-plots}
\end{figure}

First, we see that the constrained adaptive $\onell$~GA~-- even with the logarithmic adaptation constraint, which has not (yet) been proven to be helpful~-- manages to perform better than all fixed-size variations. However, its runtime is not linear here, which may indicate that the logarithmic constraint is too big to be convenient for this problem.

Second, and the most noticeable, the $\onell$~GA with unlimited adaptation performs even worse than the plain $(1+1)$~EA. The reasons for this behavior is that the values of population size $\lambda = \Theta(\sqrt{n})$, common at the latest stages of optimization, which also dictate the mutation probability, interfere in bad ways with the problem. In other words, conditions of Theorem~\ref{3cnf-bounding-theorem} are severely violated, which destroys the adaptation logic of the $\onell$~GA.

We performed an additional series of runs, where for each iteration of each run we recorded the distance to the optimum $d(x)$ of the currently best individual $x$, as well as the value of $\lambda$. The sizes to be considered were $n = 2^t$ for $t \in [13..16]$, and the number of clauses $m$ was, again, $m(n) = \lfloor 4 n \log n \rfloor$. For every problem size, we made five runs. The plots for dependencies of $\lambda$ on $d = d(x)$ are presented on Fig.~\ref{escape-13}--\ref{escape-16}. For convenience, the values of $\sqrt{n / d}$ were taken instead of $d(x)$ to use with the abscissa axis, and both axes are logarithmic. In these coordinates, the optimal $\lambda = \lambda(d) = \sqrt{n/d}$ values form a straight line, which is drawn in black color on every plot.

\begin{figure}
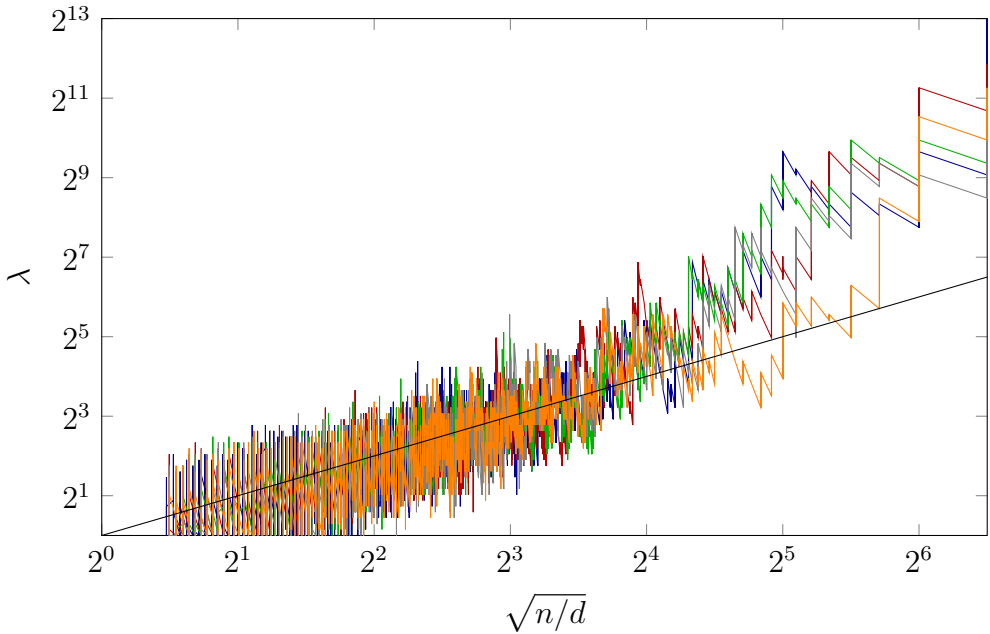

\escapeThirteen{0.85\textwidth}{0.5\textwidth}
\caption{Example runs with unconstrained $\lambda$, $n = 2^{13}$}\label{escape-13}
\end{figure}

\begin{figure}
\escapeFourteen{0.85\textwidth}{0.5\textwidth}
\caption{Example runs with unconstrained $\lambda$, $n = 2^{14}$}\label{escape-14}
\end{figure}

\begin{figure}
\escapeFifteen{0.85\textwidth}{0.5\textwidth}
\caption{Example runs with unconstrained $\lambda$, $n = 2^{15}$}\label{escape-15}
\end{figure}

\begin{figure}
\escapeSixteen{0.85\textwidth}{0.5\textwidth}
\caption{Example runs with unconstrained $\lambda$, $n = 2^{16}$}\label{escape-16}
\end{figure}

Fig.~\ref{escape-13}--\ref{escape-16} demonstrate that for distances which satisfy $\sqrt{n/d} \ge \log n$, the plots stay in a certain stripe around the optimal values for $\lambda$. However, for smaller distances, adaptation starts to diverge, and values of $\lambda$ tend to be greater than necessary. The last iterations sometimes feature very large population sizes. In Table~\ref{max-lambdas}, medians for the maximum $\lambda$ values, as well as interquartile ranges, which were seen in the main experiment, are presented. These values suggest that maximum $\lambda$ tends to be $\Theta(n)$.

\begin{table}[!t]
\caption{Medians of maximum $\lambda$ values for unconstrained adaptation}\label{max-lambdas}
\centering
\begin{tabular}{rrrr}
$\log_2 n$ & \multicolumn{1}{c}{$n$} & \multicolumn{1}{c}{median} & \multicolumn{1}{c}{IQR} \\\hline
$13$ & $8192$  & $2109.99$ & $2589.63$ \\
$14$ & $16384$ & $4747.47$ & $4236.71$ \\
$15$ & $32768$ & $10681.82$ & $11148.18$ \\
$16$ & $65536$ & $27945.01$ & $19246.74$ \\
\end{tabular}
\end{table}

%% file: parts/6-conclusion.tex
\section{Conclusion}\label{sec-conclusion}

The runtime analysis of the $\onell$-GA conducted in this work shows that the GA can cope with a weaker fitness-distance correlation than the perfect one of the \onemax test function, the only example mathematically analyzed before. However, a weaker fitness-distance correlation requires that the population size $\lambda$ is not taken too large, as otherwise the strong mutation rate of $\lambda/n$ creates offspring that are too far from each other for the GA to find the best one (in terms of distance to the optimum) in the intermediate selection step.

Our recommendation on how to use the $\onell$-GA therefore is to first try a moderate size constant $\lambda$, say $\lambda =5$ or $\lambda = 10$. If this leads to an improved performance, than larger values of $\lambda$ can be tried. For the self-adjusting version of the GA, we generally recommend using an upper limit for the value which $\lambda$ can take. For first experiments, this value should be taken around the best static value for $\lambda$ and then slowly increased.

We remark that the main part of the body of the GA, namely the generation of $y$ from $x$, can be used as a mutation operator also in other algorithms. We have no experience with this approach so far, but are optimistic that it can give good results as well.

%% file: parts/7-acknowledgments.tex
\subsection*{Acknowledgments}

This research was started when Maxim Buzdalov stayed at \'Ecole Polytechnique in May 2016 supported by a grant from the French Consulate in Russia (bourse Metchnikov). Maxim Buzdalov was also financially supported by the Government of Russian Federation, Grant 074-U01. This research also benefited from the support of the ``FMJH Program Gaspard Monge in optimization and operations research'', and from the support to this program from EDF.